\documentclass{article}

\usepackage{microtype}
\usepackage{subfig}
\usepackage{graphicx}
\usepackage{booktabs} 

\usepackage[backref=page]{hyperref}     
\renewcommand*{\backref}[1]{}
\renewcommand*{\backrefalt}[4]{%
    \ifcase #1 (Not cited.)%
    \or        (Cited on page~#2.)%
    \else      (Cited on pages~#2.)%
    \fi}

\usepackage{algorithm}
\usepackage{multirow}

\usepackage[accepted]{icml2025}

\usepackage{amsmath}
\usepackage{amssymb}
\usepackage{mathtools}
\usepackage{amsthm}

\newcommand{\std}[1]{\,\text{\scriptsize$\pm$\,#1}}

\usepackage[capitalize,noabbrev]{cleveref}

\usepackage{amsthm,amsmath,amsfonts,bm}

\def\eqref#1{equation~\ref{#1}}

\def\1{\bm{1}}

\def\eps{{\epsilon}}

\DeclareMathAlphabet{\mathsfit}{\encodingdefault}{\sfdefault}{m}{sl}
\SetMathAlphabet{\mathsfit}{bold}{\encodingdefault}{\sfdefault}{bx}{n}

\newcommand{\E}{\mathbb{E}}

\newcommand{\R}{\mathbb{R}}

\newcommand{\Var}{\mathrm{Var}}

\DeclareMathOperator*{\argmin}{arg\,min}

\DeclarePairedDelimiter{\norm}{\Vert}{\Vert}

\newcommand{\mc}[1]{\mathcal{#1}}

\def\ddefloop#1{\ifx\ddefloop#1\else\ddef{#1}\expandafter\ddefloop\fi}
\def\ddef#1{\expandafter\def\csname 
bb#1\endcsname{\ensuremath{\mathbb{#1}}}}
\ddefloop ABCDEFGHIJKLMNOPQRSTUVWXYZ\ddefloop
\def\ddefloop#1{\ifx\ddefloop#1\else\ddef{#1}\expandafter\ddefloop\fi}
\def\ddef#1{\expandafter\def\csname 
b#1\endcsname{\ensuremath{\mathbf{#1}}}}
\ddefloop ABCDEFGHIJKLMNOPQRSTUVWXYZ\ddefloop
\def\ddef#1{\expandafter\def\csname 
c#1\endcsname{\ensuremath{\mathcal{#1}}}}
\ddefloop ABCDEFGHIJKLMNOPQRSTUVWXYZ\ddefloop
\def\ddef#1{\expandafter\def\csname 
h#1\endcsname{\ensuremath{\widehat{#1}}}}
\ddefloop ABCDEFGHIJKLMNOPQRSTUVWXYZ\ddefloop
\def\ddef#1{\expandafter\def\csname 
hc#1\endcsname{\ensuremath{\widehat{\mathcal{#1}}}}}
\ddefloop ABCDEFGHIJKLMNOPQRSTUVWXYZ\ddefloop
\def\ddef#1{\expandafter\def\csname 
t#1\endcsname{\ensuremath{\widetilde{#1}}}}
\ddefloop ABCDEFGHIJKLMNOPQRSTUVWXYZ\ddefloop
\def\ddef#1{\expandafter\def\csname 
tc#1\endcsname{\ensuremath{\widetilde{\mathcal{#1}}}}}
\ddefloop ABCDEFGHIJKLMNOPQRSTUVWXYZ\ddefloop

\usepackage{nicefrac}

\newsavebox\CBox

\newcommand{\sgn}{\mathrm{sgn}}

\newcommand{\pars}[1]{\left( #1 \right)}
\newcommand{\abs}[1]{\left| #1 \right|}
\newcommand{\bracks}[1]{\left[ #1 \right]}
\newcommand{\set}[1]{\left\{ #1 \right\}}
\newcommand{\inner}[2]{\left\langle #1,\, #2 \right\rangle}

\newcommand{\removed}[1]{}

\newcommand{\EE}[1]{\mathbb{E}\left[#1\right]}

\newcommand{\grad}{\nabla}

\newcommand{\diag}{\textrm{diag}}

\usepackage{booktabs}
\usepackage{subcaption}
\DeclareSymbolFont{bbold}{U}{bbold}{m}{n}
\DeclareSymbolFontAlphabet{\mathbbold}{bbold}

\newtheorem{theorem}{Theorem}
\newtheorem{lemma}[theorem]{Lemma}

\newtheorem{corollary}{Corollary}

\newtheorem{assumption}{Assumption}
\usepackage{thm-restate} 

\usepackage{soul}
\usepackage{adjustbox}

\theoremstyle{plain}

\usepackage[textsize=tiny]{todonotes}

\icmltitlerunning{\textsc{Stacey}: Promoting Stochastic Steepest Descent via Accelerated $\ell_p$-Smooth Nonconvex Optimization}

\begin{document}

\twocolumn[
\icmltitle{\textsc{Stacey}: Promoting Stochastic Steepest Descent via Accelerated\\ \texorpdfstring{$\ell_p$}{lp}-Smooth Nonconvex Optimization}



\icmlsetsymbol{equal}{*}

\begin{icmlauthorlist}
\icmlauthor{Xinyu Luo}{equal,purdue}
\icmlauthor{Cedar Site Bai}{equal,purdue}
\icmlauthor{Bolian Li}{equal,purdue}
\icmlauthor{Petros Drineas}{purdue}
\icmlauthor{Ruqi Zhang}{purdue}
\icmlauthor{Brian Bullins}{purdue}
\end{icmlauthorlist}

\icmlaffiliation{purdue}{Department of Computer Science, Purdue University, Indiana, USA}

\icmlcorrespondingauthor{Xinyu Luo}{luo466@purdue.edu}
\icmlcorrespondingauthor{Cedar Site Bai}{bai123@purdue.edu}
\icmlcorrespondingauthor{Bolian Li}{li4468@purdue.edu}

\icmlkeywords{Machine Learning, ICML}

\vskip 0.3in
]



\printAffiliationsAndNotice{\icmlEqualContribution} 

\begin{abstract}

While popular optimization methods such as SGD, AdamW, and Lion depend on steepest descent updates in either $\ell_2$ or $\ell_\infty$ norms, there remains a critical gap in handling the non-Euclidean structure observed in modern deep networks training.
In this work, we address this need by introducing a new \emph{accelerated} $\ell_p$ steepest descent algorithm, called \textsc{Stacey}, which uses interpolated primal-dual iterate sequences to effectively navigate non-Euclidean smooth optimization tasks.
In addition to providing novel theoretical guarantees for the foundations of our algorithm, we empirically compare our approach against these popular methods
on tasks including image classification and language model (LLM) pretraining, demonstrating both faster convergence and higher final accuracy. We further evaluate different values of $p$ across various models and datasets, underscoring the importance and efficiency of non-Euclidean approaches over standard Euclidean methods. Code can be found at \url{https://github.com/xinyuluo8561/Stacey}.
\end{abstract}

\section{Introduction}
Stochastic first-order methods have proven essential for efficiently training modern deep learning models. Beyond the basic stochastic gradient descent (SGD) algorithm~\citep{robbins1951stochastic} and its momentum-based variants~\citep{nesterov1983method, polyak1964some}, a variety of adaptive methods have been developed, such as AdaGrad~\citep{duchi2011adaptive}, Adam~\citep{kingma2014adam}, and AdamW~\citep{loshchilovdecoupled}, which incorporate second-moment gradient information to provide per-coordinate scaling. Meanwhile, more recent methods like signSGD~\citep{bernstein2018signsgd} and Lion~\citep{chen2024symbolic}
focus on using the \emph{sign} of the (stochastic) gradient.

Although these algorithms have shown impressive empirical performance (sometimes exceeding that of standard adaptive methods), their theoretical analyses typically rely on Euclidean (i.e., $\ell_2$) or $\ell_\infty$-based assumptions.
Specifically, crucial to guarantees of finding $\epsilon$-approximate stationary points~\citep{carmon2017convex, ghadimi2013stochastic, jin2017escape} are two related choices: (i) the norm used to define stationarity, and (ii) the corresponding notion of smoothness. Classical analyses in deep learning often adopt Euclidean smoothness~\citep{ghadimi2013stochastic}, while signSGD relies on $\ell_\infty$-based assumptions~\citep{bernstein2018signsgd, balles2020geometry}.

Yet, there is mounting evidence—both theoretical and empirical—suggesting that a more flexible $\ell_p$ perspective can capture the geometric structure of complex deep network objectives far better than \emph{either} $p=2$ or $p=\infty$ alone~\citep{adolphs2019ellipsoidal, cohen2021gradient, ghorbani2019investigation, jiang2024does, li2020hessian, papyan2018full}. For instance, depending on the shape of the loss surface and the distribution of gradients across coordinates, certain $\ell_p$ norms with $2 < p < \infty$ may lead to faster descent or improved generalization.
This leaves open a significant gap: 
How can we develop and analyze optimizers in \emph{alternative} non-Euclidean regimes, namely those with \emph{general} $\ell_p$ norms where $p\in (2, \infty)$?

To address this question, we propose a novel approach that we term \textsc{Stacey} (\textbf{St}ochastic \textbf{St}eepest Descent with \textbf{Ac}c\textbf{e}leration). Our development builds on insights from both $\ell_p$-steepest descent and non-Euclidean acceleration techniques~\citep{allen2017linear, diakonikolas2024complementary, nemirovskii1985optimal, nesterov2005smooth}, combining primal-dual~\citep{diakonikolas2019approximate} iterates with an interpolation scheme designed specifically for $\ell_p$-based smoothness. While the notion of acceleration is well understood in the classical (Euclidean) setting~\citep{nesterov1983method}, extending it to arbitrary $\ell_p$ norms introduces a fundamental trade-off: although we may attain improved geometry dependence (and thus potentially faster practical convergence in certain regimes), the theoretical “acceleration exponent” necessarily decreases from $2$ toward $1$ as $p$ grows large~\citep{guzman2015lower, nemirovskii1985optimal}. Nonetheless, by situating \textsc{Stacey} within this continuum of non-Euclidean optimizers, we can reap meaningful benefits over purely Euclidean (e.g., SGD) and purely sign-based (e.g., signSGD) methods on modern, large-scale tasks.

\paragraph{Our Contributions:} 
\begin{itemize}
\item \textbf{Accelerated $\ell_p$-based method.} Drawing inspiration from primal-dual interpolation techniques in the convex setting~\citep{allen2017linear, nesterov2005smooth, diakonikolas2024complementary}, we design \textsc{Stacey}, an \emph{accelerated} $\ell_p$ descent algorithm specifically tailored to non-Euclidean smooth optimization (Section~\ref{sec:accel}).
\item \textbf{General $\ell_p$ convergence guarantees for non-convex problems.} We first establish \emph{stochastic $\ell_p$ steepest descent} guarantees $\EE{\norm{\grad f(\hat{x})}_{p^*}^{p^*}} \leq \eps$ at a rate of $O(\epsilon^{-4})$, under standard variance-bounded and $\ell_p$-smoothness assumptions, where we let $p^* := \frac{p}{p-1}$  (Section~\ref{sec:lp_converge}). Our results strictly generalize previous guarantees for signSGD ($p=\infty$).
\item \textbf{Practical performance on large-scale tasks.} We compare \textsc{Stacey} against popular optimizers such as SGD, Adam, AdamW, and Lion on tasks ranging from image classification to pretraining large language models (Section~\ref{sec:experiments}). Our experiments show that \textsc{Stacey} can converge faster and achieve higher accuracy than these baselines, particularly when the geometry of the objective departs significantly from the Euclidean setting.
\item \textbf{Flexible norm choices.} We further evaluate different values of $p\in (2, \infty)$ across various model architectures and datasets, illustrating the potential advantages of tailoring the choice of norm to the problem geometry.
\end{itemize}

Taken together, our results highlight the importance of \emph{non-Euclidean} perspectives for contemporary machine learning tasks, offering both theoretical insight and practical improvement over classical ($\ell_2$-based) and sign-based ($\ell_\infty$-based) optimizers.

\section{Related Work}
\paragraph{Methods for non-Euclidean geometries.}
A significant line of research has studied \emph{sign-based} methods, which can be viewed as (stochastic) steepest descent under the $\ell_\infty$ norm. For instance, \citet{bernstein2018signsgd} introduced signSGD and analyzed its convergence properties through an $\ell_2$ majorization-based smoothness condition,\footnote{We provide a comparison of $\ell_2$ majorization and $\ell_p$ smoothness conditions in Appendix~\ref{app:l2maj}.} showing that in expectation, $\|\nabla f(\hat{x})\|_1$ can be driven below a prescribed threshold. Similarly, \citet{balles2020geometry} investigated the geometric underpinnings of sign-based updates, highlighting how they relate to $\ell_\infty$ steepest descent. Recent work on Lion~\citep{chen2024symbolic} and its generalization Lion-$\mathcal{K}$~\citep{chen2024lion} further underscores the empirical benefits of sign-driven coordinates in large-scale tasks.

However, sign-based approaches ($p = \infty$) represent just one extreme of non-Euclidean geometry. The other well-studied example is the classical $\ell_2$-based regime (e.g., vanilla SGD)~\citep{ghadimi2013stochastic, robbins1951stochastic}, where standard notions of Euclidean smoothness and approximate stationarity $\|\nabla f(\hat{x})\|_2 \le \epsilon$ underpin core theoretical results. Interpolating between these extremes ($\ell_2$ and $\ell_\infty$) by considering $\ell_p$ norms for $2 < p < \infty$ has remained comparatively underexplored in the stochastic, non-convex setting. One challenge is that, unlike in the Euclidean case, the coordinate-scaling in an $\ell_p$ steepest-descent update is not merely a straightforward unbiased estimator of the full-batch direction, making a standard “SGD-style” analysis more involved.

\paragraph{Methods for curvature-aware optimization.} Another line of work exploits local geometry by incorporating second-order information, such as the Hessian or Fisher information matrix, and develops techniques for their efficient approximation. K-FAC \citep{martens2015optimizing} approximates the Fisher information matrix using layer-wise Kronecker-factored preconditioners for efficient second-order updates. Shampoo \citep{gupta2018shampoo,morwani2025a,vyas2025soap} similarly employs per-dimension Kronecker-factored preconditioners to approximate the gradient’s second-moment matrix, enabling scalable curvature-aware optimization for tensor-structured parameters. Sophia \citep{liu2024sophia} further improves scalability by approximating the diagonal of the Hessian using second-order momentum. In contrast, our method, {\sc Stacey}, is a first-order approach that leverages non-Euclidean geometry, rather than local curvature, through a differing $\ell_p$ norm.

\paragraph{Why $\ell_p$-based methods help for large models.} 
A key motivation for exploring $\ell_p$-norms with $p \in (2,\infty)$ stems from recent studies on the Hessian spectrum of large neural networks~\citep{ghorbani2019investigation, papyan2018full}. In particular, \citet{ghorbani2019investigation} provide evidence that the Hessian eigenvalue density can be highly non-uniform, leading to large curvature in certain subspaces while others remain comparatively flat. Under standard $\ell_2$-based (Euclidean) assumptions, these directions of high curvature can inflate the global smoothness parameter $L_2$, potentially slowing convergence or complicating optimization. By transitioning to $\ell_p$-smoothness for $p>2$, one can sometimes leverage a reasonable Lipschitz constant $L_p$, as high-curvature directions may not always dominate in the same way.

Formally, a function $f$ is $\ell_p$-smooth if, for all $x, y \in \R^d$,
$$\|\nabla f(x) - \nabla f(y)\|_{p^*} \;\leq\; L_p\,\|x-y\|_{p},$$
where $p^* = \frac{p}{p-1}$ is the dual value. Thus, the choice of $p$ shifts how curvature in different coordinates or subspaces affects $\nabla f$. Because large-scale models often exhibit anisotropic Hessians~\citep{cohen2021gradient, li2020hessian}, an $\ell_p$ analysis can better mirror the true geometric structure of the objective. This observation aligns with analyses in \cite{balles2020geometry}, where sign-based methods (i.e., $\ell_\infty$) can exploit flat directions effectively; by continuity, $\ell_p$ norms for $p\in(2,\infty$) may interpolate between purely Euclidean and purely sign-driven behaviors.

Two main factors motivate the study of general $\ell_p$-norms ($2 < p < \infty$) in large-scale training:

\begin{enumerate}
\item \textbf{Hessian Geometry and Tail Behavior.} Large neural networks often exhibit Hessians whose eigenvalues and singular vectors follow nontrivial (sometimes heavy-tailed) distributions \citep{ghorbani2019investigation, papyan2018full}. By choosing $p$ to better accommodate outlier directions or to exploit more uniform curvature across coordinates, one can leverage better effective $\ell_p$-smoothness constant $L_p$.
\item \textbf{Balancing Sparse and Dense Updates.} Methods at $p=\infty$ (sign-based) produce coordinate-wise updates of the same magnitude, while $\ell_2$-based approaches “spread out” updates proportionally to gradient magnitudes. In high dimensions, intermediate $\ell_p$ steps can yield a better trade-off between these extremes, potentially improving both speed of descent and generalization \citep{cohen2021gradient, li2020hessian}.
\end{enumerate}

\paragraph{Trade-offs for non-Euclidean acceleration.} 
Alongside this matter of defining (and parameterizing) smoothness, there is a second lens through which we observe the potential for general $p$, \emph{namely that of acceleration}~\citep{allen2017linear, bai2024faster, nemirovskii1985optimal, nesterov1983method, nesterov2005smooth}. As we further discuss in Section~\ref{sec:accel}, there is a fundamental trade-off (for convex settings) between the rate of acceleration and the norm used to measure the initial distance to the optimal solution. Concretely, it is well known that, for convex $f(x)$ that is $L$-smooth with respect to $\norm{\cdot}_2$, the classic accelerated gradient descent (AGD) method of~\cite{nesterov1983method} converges at the rate $f(x_T)-f(x^*) \leq O\left(\frac{L\norm{x_0 - x^*}_2^2}{T^2}\right)$, and this rate is indeed tight~\citep{nesterov2018lectures, nemirovskij1983problem}. Importantly, we emphasize the appearance here of $\norm{\cdot}_2$ not only in measuring smoothness, but also for the $\norm{x_0 - x^*}_2^2$ term.

Unfortunately, the standard analysis of AGD does not readily adapt to alternative notions of smoothness, as the design of the algorithm is, in a sense, \emph{specific to Euclidean settings}; we refer the reader to the work of~\cite{allen2017linear} for further discussion of this basic incompatibility. Nevertheless, several works~\citep{diakonikolas2024complementary, nemirovskii1985optimal, nesterov2005smooth, song2021unified}---including that of \citet{allen2017linear}---with optimal rates in the Euclidean setting \citep{nesterov2018lectures, bai2025tight}, have provided techniques for \emph{accelerating in non-Euclidean settings}.
In particular, the approach of~\citet{nemirovskii1985optimal}, for convex $f(x)$ that is $L_p$-smooth with respect to $\norm{\cdot}_p$, leads to guarantees of the form 
\begin{equation}\label{eq:pacc}
    f(x_T)-f(x^*) \leq O\left(\frac{L_p\norm{x_0 - x^*}_p^2}{T^\frac{p+2}{p}}\right).
\end{equation}
(See also, e.g., Theorem 2 in~\citep{diakonikolas2024complementary}.) Moreover, these rates are likewise known to be tight~\citep{guzman2015lower}.

Looking closely at these convergence guarantees, we may first note that, for $p=2$, the rate in \eqref{eq:pacc} recovers that of~\citet{nesterov1983method}. On the other hand, for $p \rightarrow \infty$, while $\norm{x_0 - x^*}_p^2$ can, at best, be as small as $d^{\frac{2}{p}-1}\norm{x_0 - x^*}_2^2$, we also have that $\lim_{p\to\infty}T^{-\frac{p+2}{p}} = T^{-1}$---\emph{in which case the benefit of acceleration disappears altogether}---and in fact this (limiting) rate essentially matches that of \emph{unaccelerated} $\ell_\infty$ steepest descent~\citep{kelner2014almost}.
Consequently, these observations reveal the opportunity afforded by (non-Euclidean) $\ell_p$-based accelerated methods \textbf{for other values of  $p < \infty$}, resulting from this trade-off between the \emph{dependence on the problem geometry} and the \emph{rate of acceleration}.

\section{Preliminaries and Assumptions}\label{sec:prelims}
Throughout we let $\norm{\cdot}$ and $\norm{\cdot}_*$ denote a general norm and its dual, respectively. In addition, we specify $\norm{\cdot}_p$ to denote the standard $\ell_p$ norm ($1 \leq p \leq \infty$) and $\norm{\cdot}_{p^*} := \norm{\cdot}_{p/(p-1)}$ to denote its dual norm. For symmetric $M \in \R^{d \times d}$ s.t. $M \succ 0$, we further let $\norm{\cdot}_M$ denote the standard matrix norm, i.e., $\norm{x}_M = \sqrt{x^\top M x}$ for $x \in \R^d$. For a vector $v \in \R^d$, we use superscript, i.e., $v^{(i)}$ to denote the $i^{th}$ coordinate of $v$, and we let $\diag(v)$ denote the diagonal matrix such that $\diag(v)_{i,i} = v^{(i)}$. We use subscript, e.g., $\theta_t$, to denote a vector in the $t^{th}$ iteration. For brevity, we use $g_t$ for the true gradient $\nabla f(\theta_t)$ and $\tilde{g}_t$ for the stochastic gradient $g(\theta_t)$. We use $\sgn\pars{\cdot}$ to denote the sign function and $\mathbb{I}_{\bracks{\cdot}}$ to denote the indicator function.

We may also consider the following equivalent definition of $\ell_p$ smoothness.
\begin{assumption}[Smoothness in $\ell_p$ norm] \label{asm:smooth-lp}
     Let $f : \R^d \mapsto \R$ be $L$-smooth w.r.t. $\norm{\cdot}_p$ for $p \geq 2$. Then, for all $x, y \in \R^d$,
    \begin{equation*}
        \abs{f(y) - f(x) - \grad f(x)^\top (y-x)} \leq \frac{L}{2}\norm{y-x}_p^2\ .
    \end{equation*}
\end{assumption}

\section{Accelerating Stochastic Steepest Descent}\label{sec:accel}


\begin{algorithm}[t]
\caption{$\textsc{Stacey}_{(p, 2)}$ Optimizer} \label{alg:p2}
    \textbf{input} $p, \beta_1, \beta_2, \alpha, \tau, \eta, \epsilon, \lambda, f$ \\
    \textbf{initialize} $\theta_0, z_0, m_0 \leftarrow 0$
    \begin{algorithmic}[1]
        \WHILE{$\theta_{t+1}$ not converged} 
        \STATE $\tilde{g}_t \leftarrow \tilde{g}$ s.t. $\mathbb{E}[\tilde{g}] = \grad f(\theta_t)$
        \STATE $c_{t+1} \leftarrow \beta_1 m_t + (1-\beta_1)\tilde{g}_t$   
        \STATE $s^\epsilon(x) = [s^\epsilon_1(x), \cdots, s^\epsilon_d(x)]^\top$ where $$s_i^\epsilon\pars{x} = \frac{x^{(i)}}{\abs{x^{(i)}}^{\frac{p-2}{p-1}}+\epsilon}, \ \forall \ i \in [d]$$
        \STATE $y_{t+1} \leftarrow \theta_t - \eta_t s_\epsilon\pars{c_{t+1}}$ 
        \STATE $z_{t+1} = z_t - \alpha c_{t+1}$
        \STATE $\theta_{t+1} = \tau z_{t+1} + (1-\tau)y_{t+1} - \eta_t \lambda \theta_t$
        \STATE $m_{t+1} = \beta_2 m_t + (1-\beta_2)\tilde{g}_t$
        \ENDWHILE
        \STATE \textbf{return} $\theta_{t+1}$
    \end{algorithmic}
\end{algorithm}

Inspired by previous techniques in non-Euclidean acceleration~\citep{allen2017linear, nesterov2005smooth}---as well as their successes, e.g., ~\citep{bullins2020highly, jambulapati2019direct, sherman2017area, sidford2018coordinate}---we introduce a practical acceleration scheme called {\sc Stacey} (Algorithm \ref{alg:p2}), which is \emph{specifically designed for $\ell_p$-based methods}. Central to our approach is its reliance on \emph{two} sequences of stochastic steps: 1) one sequence based on the standard $\ell_p$ steepest descent direction (line 5), which we show is theoretically well-grounded in the stochastic non-convex setting; 2) another sequence---whose combination with the first ultimately leads to acceleration---based on a gradient descent direction (line 6), whose details we will further discuss.

\begin{algorithm}[t]
\caption{$\textsc{Stacey}_{(p, p)}$ Optimizer} \label{alg:pp}
    \textbf{input} $p, \beta_1, \beta_2, \alpha, \tau, \eta, \epsilon, \lambda, f$ \\
    \textbf{initialize} $\theta_0, z_0, m_0 \leftarrow 0$
    \begin{algorithmic}[1]
        \WHILE{$\theta_{t+1}$ not converged} 
        \STATE $\tilde{g}_t \leftarrow \tilde{g}$ s.t. $\mathbb{E}[\tilde{g}] = \grad f(\theta_t)$
        \STATE $c_{t+1} \leftarrow \beta_1 m_t + (1-\beta_1)\tilde{g}_t$ 
        \STATE $s^\epsilon(x) = [s^\epsilon_1(x), \cdots, s^\epsilon_d(x)]^\top$ where $$s_i^\epsilon\pars{x} = \frac{x^{(i)}}{\abs{x^{(i)}}^{\frac{p-2}{p-1}}+\epsilon},\ \forall \ i \in [d]$$
        \STATE $y_{t+1} \leftarrow \theta_t - \eta_t s^\epsilon\pars{c_{t+1}}$ 
        \STATE $z^{(i)}_{t+1} = \frac{\abs{z^{(i)}_t}^{p-2}z^{(i)}_t - \alpha c^{(i)}_{t+1}}{\abs{\abs{z^{(i)}_t}^{p-2}z^{(i)}_t - \alpha c^{(i)}_{t+1}}^{\frac{p-2}{p-1}}+\epsilon}$, $\forall \ i \in [d]$
        \STATE $\theta_{t+1} = \tau z_{t+1} + (1-\tau)y_{t+1} - \eta_t \lambda \theta_t$
        \STATE $m_{t+1} = \beta_2 m_t + (1-\beta_2)\tilde{g}_t$
        \ENDWHILE
        \STATE \textbf{return} $\theta_{t+1}$
    \end{algorithmic}
\end{algorithm}

\subsection{Analyzing Stochastic \texorpdfstring{$\ell_p$}{lp} Descent}\label{sec:lp_converge}


In this section, we present the stochastic $\ell_p$ descent algorithm, which serves as the fundamental framework of our approach, and establish its convergence guarantees. As shown in Algorithm \ref{alg:ellp-descent}, its update step takes the unscaled form\footnote{This is in line with signSGD \citep{bernstein2018signsgd} compared to the scaled form in \citep{balles2020geometry}. In addition, we adopt the unscaled version for clearer convergence analysis and a more practical implementation.} of its counterpart in the deterministic setting $\theta_{t+1}^{(i)} = \theta_{t}^{(i)} - \eta \norm{g_t}_{p^\ast}^\frac{p-2}{p-1}\frac{g_t^{(i)}}{\abs{g_t^{(i)}}^\frac{p-2}{p-1}}$, which is derived from the closed form of \[\theta_{t+1} = \argmin_{\theta} \set{\inner{\eta g_t}{\theta - \theta_t} + \frac{1}{2}\norm{\theta - \theta_t}_p^2}.\] When $p = \infty$, Algorithm \ref{alg:ellp-descent} reduces exactly to signSGD \citep{bernstein2018signsgd}.

\begin{algorithm}[t] 
\caption{Stochastic $\ell_p$ Descent}
\label{alg:ellp-descent}
    \textbf{input} $p, \eta, f, \theta_0$
    \begin{algorithmic}[1] 
        \FOR{$t=0$ \textbf{to} $T-1$} 
        \STATE $s(x) = [s_1(x), \cdots, s_d(x)]^\top$ where $$s_i(x) = \frac{x^{(i)}}{|x^{(i)}|^{\frac{p-2}{p-1}}},\ \forall \ i \in [d]$$
        \STATE $\theta_{t+1} = \theta_t - \eta s\pars{\tilde{g}_t}$ \hfill$\triangleright\ \ \tilde{g}_t$ s.t. $\mathbb{E}[\tilde{g}_t] = \grad f(\theta_t)$
        \ENDFOR
        \STATE \textbf{return} $\theta_{T}$
    \end{algorithmic}
\end{algorithm}

For $p > 2$, we show in Theorem \ref{thm:main} that stochastic $\ell_p$ descent converges in expectation to an $\epsilon$-approximate stationary point with respect to the dual norm at a rate of $O(\eps^{-4})$, thereby generalizing the previous guarantees for signSGD ($p=\infty$). In addition, we provide here a proof sketch, deferring the complete proof to Appendix \ref{app:proof}. Curiously, as we will see, moving from the $\ell_2$ setting (or even from the $\ell_\infty$ setting) introduces certain technical considerations that need to be addressed non-trivially. As standard in stochastic and non-Euclidean settings~\cite{ghadimi2013stochastic, bernstein2018signsgd}, we rely on the following assumptions.

\begin{assumption}[Unbiased Estimate] \label{asm:unbiased-grad} The stochastic gradient $\tilde{g}$ is an unbiased estimate of the true gradient $g$. That is, $\E[\tilde{g}] = g$.
\end{assumption}
\begin{assumption} [Bounded Variance] \label{asm:coord-var} For some data $\xi$, the variance of each coordinate of the stochastic gradient is bounded, i.e., $\forall i \in [d]$, $\E[|\tilde{g}^{(i)} - g^{(i)}|^2] \leq \sigma_i^2$.
\end{assumption}
\begin{corollary} By Assumption \ref{asm:coord-var}, $\E[\norm{\tilde{g} - g}_2^2] \leq \sigma^2$ where for $\sigma \coloneqq \norm{\vec{\sigma}}_2$, $\vec{\sigma} = [\sigma_1, \cdots, \sigma_d]^\top$.
\end{corollary}
\begin{corollary} If the stochastic gradient is an $n$-sample mini-batch estimate, then
    $\forall i \in [d]$, $\E[|\tilde{g}^{(i)} - g^{(i)}|^2] \leq \frac{\sigma_i^2}{n}$.
\end{corollary}
\begin{assumption}[Bounded gradient] \label{asm:bound-grad}For $G > 0$, $p \geq 2$, and $p^\ast$ where $\frac{1}{p} + \frac{1}{p^\ast} = 1$,  $\norm{\tilde{g}}_{p^\ast} \leq G$.
\end{assumption}
\begin{corollary} \label{cor:bound-grad} By Assumption \ref{asm:bound-grad}, we know that
\begin{itemize}
    \item [(a)] $\norm{g}_{p^\ast} = \norm{\E\bracks{\tilde{g}}}_{p^\ast} \leq \E\bracks{\norm{\tilde{g}}_{p^\ast}} \leq G$ with Jensen's inequality.
    \item [(b)] $\forall \ i \in [d]$, $\abs{\tilde{g}^{(i)}} \leq G$ and $\abs{g^{(i)}} \leq G$.
\end{itemize}
    
\end{corollary}

We briefly justify the necessity of Assumption \ref{asm:bound-grad}, which arises from additional technical challenges. Specifically, the coordinate-wise re-scaled update introduces bias under standard assumptions, preventing the direct application of conventional expectation and variance analyses as we later elaborate in detail. Notably, similar assumptions are also made when analyzing problems with complex structures, such as stochastic compositional \citep{wang2017stochastic}, composite \citep{wang2024online, duchi2011dual}, and federated optimization \citep{li2020On, yuan2021federated, bai2024local}. Now we introduce the convergence result for $\ell_p$ steepest descent in the stochastic non-convex setting.



\begin{restatable}[Main]{theorem}{thmmain}   \label{thm:main} Running Algorithm \ref{alg:ellp-descent} on some (possibly non-convex) function $f$ that satisfies Assumptions \ref{asm:smooth-lp} to \ref{asm:bound-grad} yields 
\begin{align*}
        \E\bracks{\frac{1}{T}\sum_{t=0}^{T-1} \norm{g_t}_{p^\ast}^{p^\ast}}& \leq \frac{f_0 - f^\ast}{\eta T} + \frac{L\eta G^\frac{2}{p-1}}{2}\\
        &\quad + \frac{1}{T}\sum_{t=0}^{T-1}\frac{\frac{2p-1}{p-1}G^\frac{1}{p-1} \norm{\vec{\sigma}}_1}{\sqrt{n_t}} 
    \end{align*}   
where $f_0 = f(\theta_0)$ and $f^\ast = f(\theta^\ast)$, $n_t$ is the batch size in iteration $t$ and $L$, $\vec{\sigma}$, and $G$ are constants from Assumption \ref{asm:smooth-lp}, \ref{asm:coord-var}, \ref{asm:bound-grad}. Further letting the batch size $n_t = T$, the number of gradient call is $N=T^2$ for $T$ iterations. With $\eta = \frac{1}{L^\frac{1}{2}G^\frac{1}{p-1}T^\frac{1}{2}}$ we have 
\begin{align*}
        &\E\bracks{\frac{1}{T}\sum_{t=0}^{T-1} \norm{g_t}_{p^\ast}^{p^\ast}} \leq \\&  \frac{1}{N^\frac{1}{4}}\bracks{L^\frac{1}{2}G^\frac{1}{p-1}\pars{f_0 - f^\ast + \frac{1}{2}} + \frac{2p-1}{p-1}G^\frac{1}{p-1} \norm{\vec{\sigma}}_1},
    \end{align*} 
   i.e., Algorithm \ref{alg:ellp-descent} takes $N \in \mc{O}\pars{\epsilon^{-4}}$ gradient queries to reach an $\epsilon$-approximate stationary point.
\end{restatable}

\textit{Proof Sketch.} Starting with Assumption \ref{asm:smooth-lp} and the descent step in Algorithm \ref{alg:ellp-descent},
    \begin{align*}
        f(\theta_{t+1}) &\leq f(\theta_t) - \underbrace{\eta\inner{g_t}{s(g_t)}}_{A} + \underbrace{\eta \inner{g_t}{s(g_t) - s(\tilde{g}_t) }}_{B} \\
        & \quad + \underbrace{\frac{L\eta^2}{2} \norm{s(\tilde{g}_t)}_p^2}_{C},
    \end{align*}
where $A = \eta\norm{g_t}_{p^\ast}^{p^\ast}$. In conventional first-order analysis, the inner product term $B$ is supposed to cancel out after taking expectation. In contrast, the closed-form stochastic $\ell_p$ descent update is coordinate-wise re-scaled, which makes the descent step \emph{biased}, that is, $\E[s(\tilde{g})] \neq s(f(x))$. In the literature on biased gradient descent \citep{stich2020analysis, demidovich2023a}, the bias terms simply accumulate as constants and do not decay with the iterations. Thus, this term requires novel techniques to guarantee convergence. Noticing that $s_i(x) = \frac{x^{(i)}}{|x^{(i)}|^{\frac{p-2}{p-1}}} = \sgn(x^{(i)})|x^{(i)}|^{\frac{1}{p-1}}$,
\resizebox{0.48\textwidth}{!}{$
\begin{aligned}
        &B
        = \eta\sum_{i=1}^d g_t^{(i)} \left(\sgn\left(g_t^{(i)}\right)|g_t^{(i)}|^{\frac{1}{p-1}} - \sgn\left(\tilde{g}_t^{(i)}\right)|\tilde{g}_t^{(i)}|^{\frac{1}{p-1}} \right) \\ 
        &= \eta\sum_{i=1}^d \abs{g_t^{(i)}} \left(|g_t^{(i)}|^{\frac{1}{p-1}} + |\tilde{g}_t^{(i)}|^{\frac{1}{p-1}}\right) \mathbb{I}_{\bracks{\sgn\left(g_t^{(i)}\right) \neq \sgn\left(\tilde{g}_t^{(i)}\right)}}  \\
        &\quad + \eta\sum_{i=1}^d \abs{g_t^{(i)}} \left||g_t^{(i)}|^{\frac{1}{p-1}} - |\tilde{g}_t^{(i)}|^{\frac{1}{p-1}}\right| \mathbb{I}_{\bracks{\sgn\left(g_t^{(i)}\right) = \sgn\left(\tilde{g}_t^{(i)}\right)}}. \\
        &\quad
    \end{aligned} 
    $} 
    Denote the first term as $B_1$ and the second $B_2$. The $|g_t^{(i)}|^{\frac{1}{p-1}} + |\tilde{g}_t^{(i)}|^{\frac{1}{p-1}}$ term in $B_1$ can be bounded by $2G^\frac{1}{p-1}$ with Corollary \ref{cor:bound-grad}, after which we take expectation, turning the indicator into a probability, and Lemma \ref{lem:B1} in Appendix \ref{app:proof} shows $\E\bracks{B_1} \leq \frac{2\eta G^\frac{1}{p-1}\norm{\vec{\sigma}}_1}{\sqrt{n_t}}$ using Markov's inequality.
    
    $B_2$ requires more sophisticated handling since we cannot push the expectation through due to the data dependence of the term $\left||g_t^{(i)}|^{\frac{1}{p-1}} - |\tilde{g}_t^{(i)}|^{\frac{1}{p-1}}\right|$, nor does $\mathbb{P}\bracks{\sgn\left(g_t^{(i)}\right) = \sgn\left(\tilde{g}_t^{(i)}\right)}$ give us much information. We instead take the zeroth-order Taylor expansion so that $\forall \ i \in [d], \ \exists \ \zeta^{(i)}$ between $g_t^{(i)}$ and $\tilde{g}_t^{(i)}$ such that 
    \resizebox{0.48\textwidth}{!}{$\begin{aligned} &{\scriptscriptstyle\quad} \\
    &|g_t^{(i)}|^{\frac{1}{p-1}} = |\tilde{g}_t^{(i)}|^{\frac{1}{p-1}} + \frac{1}{p-1}\sgn(\zeta^{(i)})\abs{\zeta^{(i)}}^\frac{2-p}{p-1} \pars{g_t^{(i)} - \tilde{g}_t^{(i)}}. \\
    &{\scriptscriptstyle\quad}\end{aligned}$}
    In addition, we have \begin{align*}&\left||g_t^{(i)}|^{\frac{1}{p-1}} - |\tilde{g}_t^{(i)}|^{\frac{1}{p-1}}\right|\\&\qquad\qquad = \frac{1}{p-1}\sgn(\zeta^{(i)})\abs{\zeta^{(i)}}^\frac{2-p}{p-1} \pars{g_t^{(i)} - \tilde{g}_t^{(i)}}.
    \end{align*}
    Furthermore, given $\sgn\left(g_t^{(i)}\right) = \sgn\left(\tilde{g}_t^{(i)}\right)$, it is either $\abs{g_t^{(i)}} \leq \abs{\zeta^{(i)}} \leq \abs{\tilde{g}_t^{(i)}}$ or $\abs{g_t^{(i)}} \geq \abs{\zeta^{(i)}} \geq \abs{\tilde{g}_t^{(i)}}$. Appendix \ref{app:proof} Lemma \ref{lem:B2} shows that $\E\bracks{B_2} \leq \frac{\eta G^\frac{1}{p-1} \norm{\vec{\sigma}}_1}{(p-1)\sqrt{n_t}}$ in either case.

    Term $C$ is usually turned into mean-squared error that coincides with variance in an unbiased setting, which the bounded variance assumption can directly handle. This is not the case for our setting. It is worth noting that the analysis of signSGD \citep{bernstein2018signsgd}, a special case of the $\ell_p$ setting with $p = \infty$, was able to push through due to its update being in the very form of the sign of the gradient, which is in itself bounded by the constant $1$. Our update, in contrast, is more complicated with the absolute value of the coordinates of the gradient in the denominator, which is only lower bounded by $0$, or some $\epsilon > 0$ at best. Therefore, we directly apply Assumption \ref{asm:bound-grad} and $C = \frac{L\eta^2}{2} \norm{g_t}_{p^\ast}^{\frac{2}{p-1}} \leq \frac{L\eta^2G^\frac{2}{p-1}}{2}$. Moving term $A$ to the left hand side, telescoping across iterations, and dividing both sides by $\eta T$ completes the proof.
\hfill $\Box$

\subsection{$\ell_p$ acceleration}
We would note that for smooth convex optimization, (deterministic) gradient descent can be accelerated to achieve a rate of $O(1/T^2)$. However, for stochastic first-order methods, it has been shown that a) in convex settings, SGD cannot improve upon the standard $O(1/\sqrt{T})$ rate when noise parameter $\sigma$ is large enough \citep{agarwal2009information}, and b) in first-order smooth \emph{non-convex} settings, \emph{SGD cannot be accelerated (in theory)} without additional assumptions (in terms of gradient norm minimization), due to known lower bounds \citep{arjevani2023lower}. Nevertheless, standard practical implementations of SGD are frequently designed to introduce \emph{some} notion of acceleration with momentum (e.g.,~\citep{bernstein2018signsgd,sutskever2013importance}), ``pushing'' the converging sequence further along the direction of previous gradients.

In contrast, we take the view of acceleration not as a ``pushing'' (in the Euclidean sense), but rather as a (dynamic) interpolation of two iterate sequences: one acting from a (primal) steepest descent perspective (line 4 Algorithm \ref{alg:p2}), while the other functions in a dual capacity (line 5 Algorithm \ref{alg:p2}). An apparent distinction is that momentum, as a separate functionality, can be applied on top of the acceleration scheme in {\sc Stacey$_{(p, 2)}$}, as demonstrated in lines 3 and 7 of Algorithm \ref{alg:p2}, for both the steepest descent and the (Euclidean) mirror descent.

A Euclidean-based two-sequence interpolation was adopted by Schedule-Free SGD/AdamW~\citep{defazio2024road}, which removes explicit learning-rate schedules while retaining strong performance.
In the realm of non-Euclidean methods, we contrast our algorithm with Lion-$\mc{K}$ \citep{chen2024lion, bernstein2018signsgd}. While at first glance it may seem that these methods may simply be a rewriting of each other (based on the choice of parameters), a closer inspection on \emph{the very first step} reveals that such is not the case:
\begin{align*}
    \text{Lion-$\mathcal{K}$: } & \theta_1 = -\eta \nabla \mc{K}\pars{(1-\beta_1) \tilde{g}_0}, \\
    \text{\sc Stacey$_{(p, 2)}$: } & \theta_1 = - (1-\tau) \eta s^\epsilon\pars{(1-\beta_1)\tilde{g}_0} \\
     & \qquad - \tau \alpha (1-\beta_1) \tilde{g}_0.
\end{align*}
where $\mc{K}(\cdot) = \norm{\cdot}_{p^*}$ and $s^\epsilon\pars{\cdot}$ is defined in Algorithm \ref{alg:p2}. 
The key difference of {\sc Stacey$_{(p, 2)}$} lies in the convex combination of a steepest descent step and a gradient descent step, whereas Lion-$\mathcal{K}$ is composed of only the steepest descent step. They coincide only when $\tau = 0$ for {\sc Stacey$_{(p, 2)}$}, i.e., completely getting rid of the ``coupling'', which then defeats the purpose of our acceleration. In addition, there is no choice of parameters for Lion-$\mathcal{K}$ to recover linear coupling. As a result, they are not iterate-equivalent, which further highlights the fundamental difference between ``momentum'' and ``acceleration'', a distinction which, crucially, does not appear in the case of standard (Euclidean) AGD, i.e., when both steepest and mirror descent steps are with respect to Euclidean norms.

Further inspired by the fact that {\sc Stacey$_{(p, 2)}$} breaks the symmetry (in primal and dual trajectories) by coupling an $\ell_p$ steepest descent step with an $\ell_2$-based mirror descent step, we present the natural variant {\sc Stacey$_{(p, p)}$} (Algorithm \ref{alg:pp}),
for which we group $\ell_p$ steepest descent with a mirror descent step having $\frac{1}{p}\norm{\cdot}_p^p$ (whose $p^{th}$-order uniform convexity is useful for non-Euclidean acceleration~\citep{adil2024convex, contreras2024non, song2021unified}) as its distance generating function.
The closed-form mirror descent update is presented in line 5 of the algorithm.

\section{Experiments}\label{sec:experiments}

\begin{table*}[t]\small
\centering
\caption{Image classification on CIFAR at the 50th, 100th, and 200th epochs. {\sc Stacey} consistently outperforms other optimizers, demonstrating both superior accuracy and faster convergence.}
\label{tab:cifar}
\begin{tabular}{l|ccc|ccc}
\toprule
\multicolumn{1}{c|}{\multirow{2}{*}{\textbf{Optimizer}}} &
\multicolumn{3}{c|}{\textbf{Training NLL $\downarrow$}} &
\multicolumn{3}{c}{\textbf{Testing ACC (\%) $\uparrow$}} \\ \cline{2-7}
\multicolumn{1}{c|}{} & @50 epoch & @100 epoch & @200 epoch & @50 epoch & @100 epoch & @200 epoch \\ \midrule
SGD w/ Momentum        & 0.0567\std{0.0017} & 0.0441\std{0.0014} & 0.0352\std{0.0012} & 91.15\std{0.30} & 92.02\std{0.24} & 92.76\std{0.13}\\
Adam
& \textbf{0.0401}\std{0.0017} & 0.0182\std{0.0017} & 0.0083\std{0.0010} & 91.69\std{0.18} & 92.13\std{0.16} & 92.66\std{0.36} \\
AdamW
& 0.0590\std{0.0010} & 0.0278\std{0.0009} & 0.0195\std{0.0015} & 90.59\std{0.37} & 91.47\std{0.42} & 92.12\std{0.07} \\
Lion~\citep{chen2024symbolic} & 0.1006\std{0.0571} & 0.2226\std{0.1410} & 0.0245\std{0.0043} & 89.38\std{2.02} & 89.19\std{1.88} & 92.15\std{0.32} \\ \hline
{\sc Stacey}$_{(p,p)}$ & 0.0423\std{0.0009} & \textbf{0.0118}\std{0.0014} & 0.0021\std{0.0011} & \textbf{91.88}\std{0.21} & \textbf{92.79}\std{0.16} & \textbf{93.79}\std{0.38} \\
{\sc Stacey}$_{(p,2)}$ & 0.0614\std{0.0031} & 0.0131\std{0.0027} & \textbf{0.0014}\std{0.0005} & 90.83\std{0.32} & 92.70\std{0.28} & 93.54\std{0.06} \\
\bottomrule
\end{tabular}
\end{table*}



\begin{table*}[t]\small
\centering
\caption{Image classification on ImageNet at the 20th, 40th, and 60th epochs. {\sc Stacey} demonstrates superior test accuracy and faster convergence compared to other optimizers.}
\label{tab:imagenet}
\begin{tabular}{l|ccc|ccc}
\toprule
\multicolumn{1}{c|}{\multirow{2}{*}{\textbf{Optimizer}}} & \multicolumn{3}{c|}{\textbf{Training NLL $\downarrow$}} & \multicolumn{3}{c}{\textbf{Testing Top-$1$ ACC (\%) $\uparrow$}} \\ \cline{2-7} 
\multicolumn{1}{c|}{}& @20 epoch  & @40 epoch & @60 epoch  & @20 epoch   & @40 epoch  & @60 epoch \\ \midrule
SGD w/ Momentum & 2.0731\std{0.0007} &  1.7926\std{0.0006} & 1.4993\std{0.0003} & 56.34\std{0.27}  & 63.54\std{0.09} & 68.81\std{0.54}  \\
AdamW
& \textbf{1.3337}\std{0.0008} &  \textbf{0.9822}\std{0.0017} & \textbf{0.7395}\std{0.0029} & 66.12\std{0.53} & 68.47\std{0.14} & 69.31\std{0.05}              \\
Lion~\citep{chen2024symbolic} & 1.3529\std{0.0007}            & 1.0948\std{0.0126}          &  0.8605\std{0.0045}          & \textbf{67.66}\std{0.03}                &  68.43\std{0.10}           &  69.62\std{0.11}             \\ \hline
{\sc Stacey}$_{(p,p)}$ & 1.4680\std{0.0150}  & 1.1636\std{0.0159} & 1.0324\std{0.0100}  & 66.93\std{0.10} & \textbf{69.15}\std{0.15} & \textbf{69.87}\std{0.14} \\
{\sc Stacey}$_{(p,2)}$  & 1.8376\std{0.0134}  & 1.3781\std{0.0187} & 1.1983\std{0.0120} & 60.89\std{0.12}   & 66.34\std{0.16}  & 67.56\std{0.15}     \\ \bottomrule
\end{tabular}
\end{table*}

\begin{table*}[t]
\scriptsize\setlength\tabcolsep{2.9pt}
\centering
\caption{Training and testing loss of LLM pre-training at a series of steps. The proposed {\sc Stacey} optimizer consistently achieves lower loss than baselines at all steps.}
\label{tab:llm}
\begin{tabular}{l|cccc|cccc}
\toprule
\multicolumn{1}{c|}{\multirow{2}{*}{\textbf{Optimizer}}} & \multicolumn{4}{c|}{\textbf{Training Loss}}                                                               & \multicolumn{4}{c}{\textbf{Testing Loss}}                                                                 \\ \cline{2-9} 
\multicolumn{1}{c|}{}                                    & @5k step                 & @10k steps               & @20k steps               & @30k steps               & @5k step                 & @10k steps               & @20k steps               & @30k steps               \\ \midrule
SGD w/ Momentum                                          & 6.6704\std{0.0129}       & 6.5205\std{0.0088}       & 6.4206\std{0.0055}       & 6.3920\std{0.0048}       & 6.6558\std{0.0131}       & 6.5150\std{0.0085}       & 6.4173\std{0.0038}       & 6.3909\std{0.0038}       \\
Adam                                                     & 6.4548\std{0.0028}       & 6.3647\std{0.0037}       & 6.2851\std{0.0030}       & 6.2485\std{0.0028}       & 6.4493\std{0.0017}       & 6.3646\std{0.0035}       & 6.2820\std{0.0037}       & 6.2480\std{0.0028}       \\
AdamW                                                    & 5.6655\std{0.0095}       & 5.5172\std{0.0081}       & 5.4401\std{0.0091}       & 5.4268\std{0.0096}       & 5.6510\std{0.0099}       & 5.5171\std{0.0080}       & 5.4350\std{0.0088}       & 5.4240\std{0.0093}       \\
Lion~\citep{chen2024symbolic}                            & 6.8722\std{0.0656}       & 6.8190\std{0.0549}       & 6.8021\std{0.0451}       & 6.7794\std{0.0425}       & 6.8624\std{0.0587}       & 6.8220\std{0.0500}       & 6.7954\std{0.0438}       & 6.7733\std{0.0413}       \\ \midrule
{\sc Stacey}$_{(p,p)}$                                   & \textbf{5.4016}\std{0.0107} & \textbf{4.9938}\std{0.0209} & \textbf{4.6492}\std{0.0112} & \textbf{4.4962}\std{0.0123} & \textbf{5.3616}\std{0.0068} & \textbf{4.9655}\std{0.0169} & \textbf{4.6372}\std{0.0116} & \textbf{4.4879}\std{0.0132} \\
{\sc Stacey}$_{(p,2)}$                                   & 6.2492\std{0.0060}       & 6.0038\std{0.0319}       & 5.7210\std{0.0363}       & 5.5841\std{0.0379}       & 6.2312\std{0.0065}       & 5.9867\std{0.0313}       & 5.7062\std{0.0375}       & 5.5755\std{0.0375}       \\ \bottomrule
\end{tabular}
\end{table*}

\begin{figure}[t]
    \centering
    \subfloat[Training loss of {\sc Stacey}$_{(p,p)}$]{\includegraphics[width=.49\columnwidth]{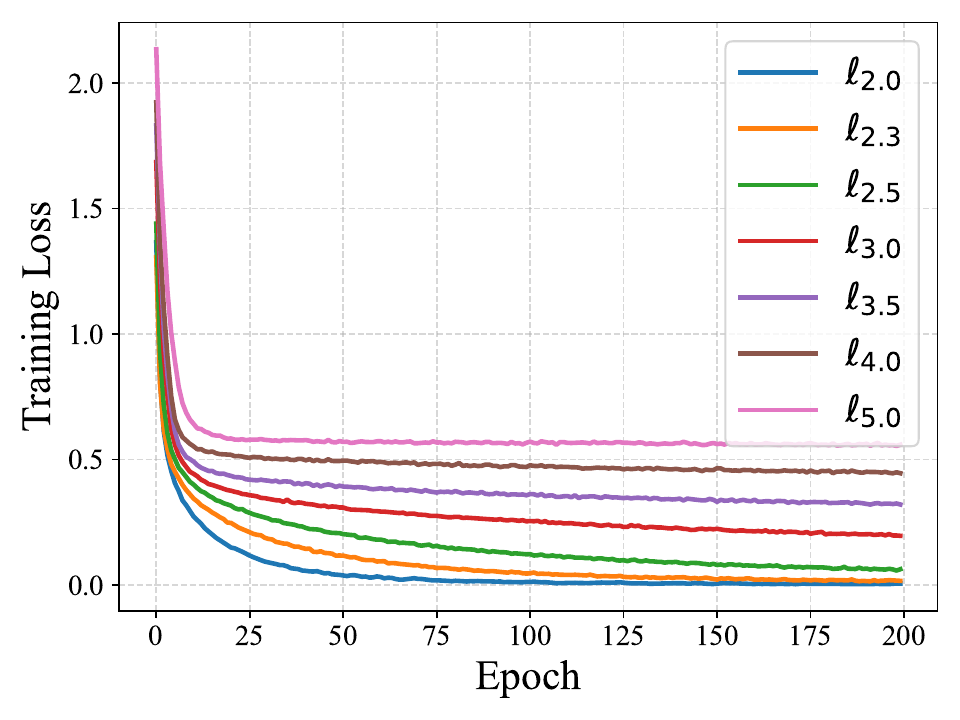}}\hspace{0pt}
    \subfloat[Testing ACC of {\sc Stacey}$_{(p,p)}$]{\includegraphics[width=.49\columnwidth]{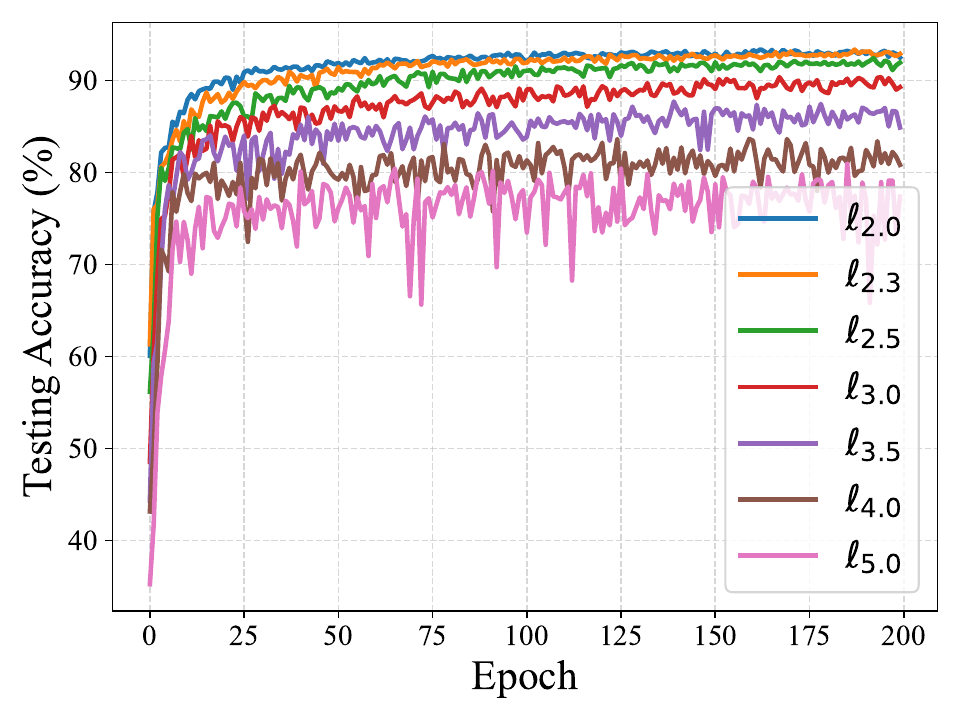}}\hspace{0pt}
    \subfloat[Training loss of {\sc Stacey}$_{(p,2)}$]{\includegraphics[width=.49\columnwidth]{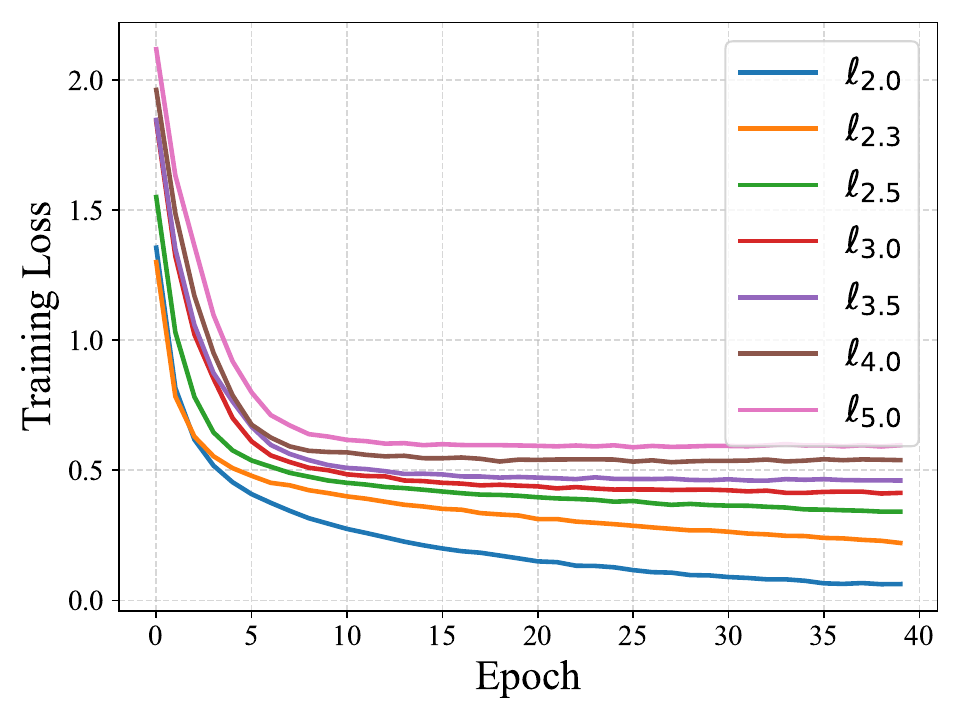}}\hspace{0pt}
    \subfloat[Testing ACC of {\sc Stacey}$_{(p,2)}$]{\includegraphics[width=.49\columnwidth]{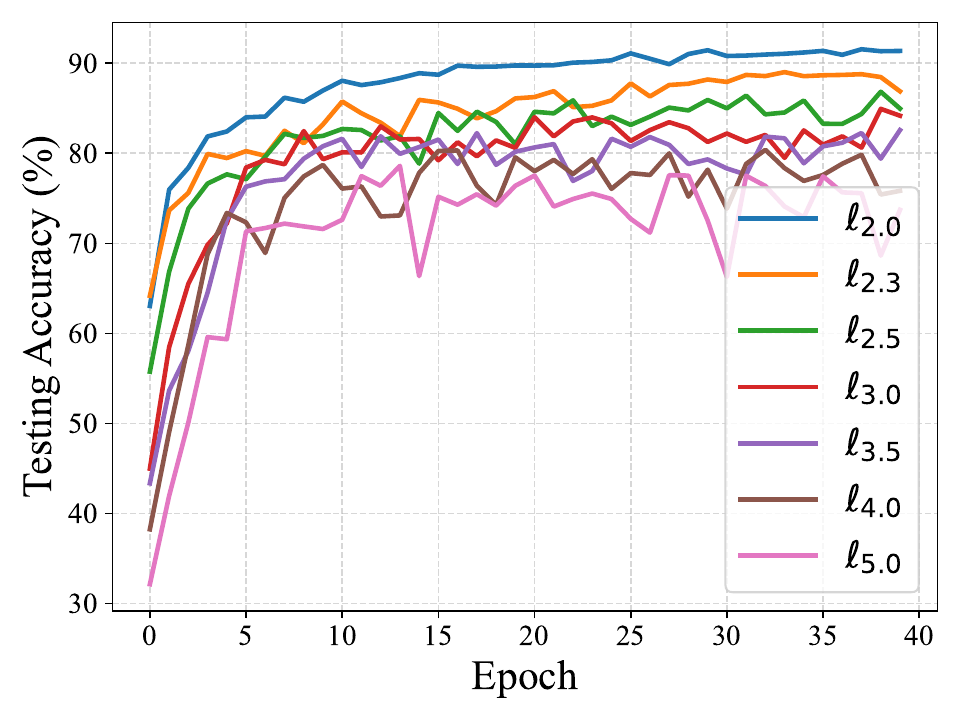}}\hspace{0pt}
    \caption{Learning curves of CIFAR classification with varying $\ell_p$-norm.}
    \label{fig:cifar_diff_p}
\end{figure}

\begin{figure}[t]
    \centering
    \subfloat[Training loss of {\sc Stacey}$_{(p,p)}$]{\includegraphics[width=.49\columnwidth]{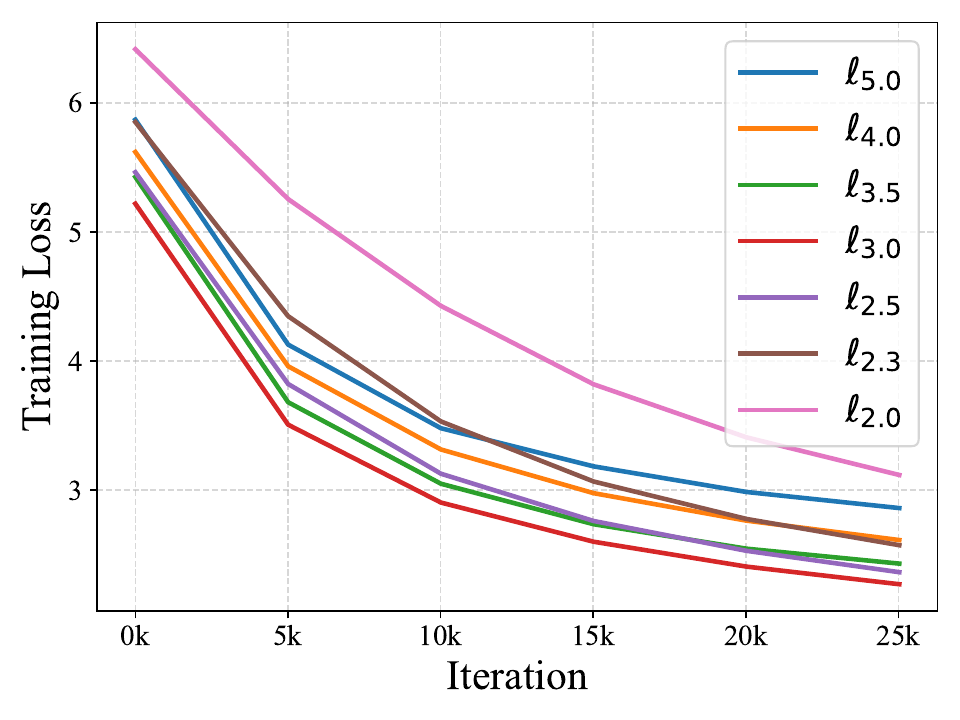}}\hspace{0pt}
    \subfloat[Testing ACC of {\sc Stacey}$_{(p,p)}$]{\includegraphics[width=.49\columnwidth]{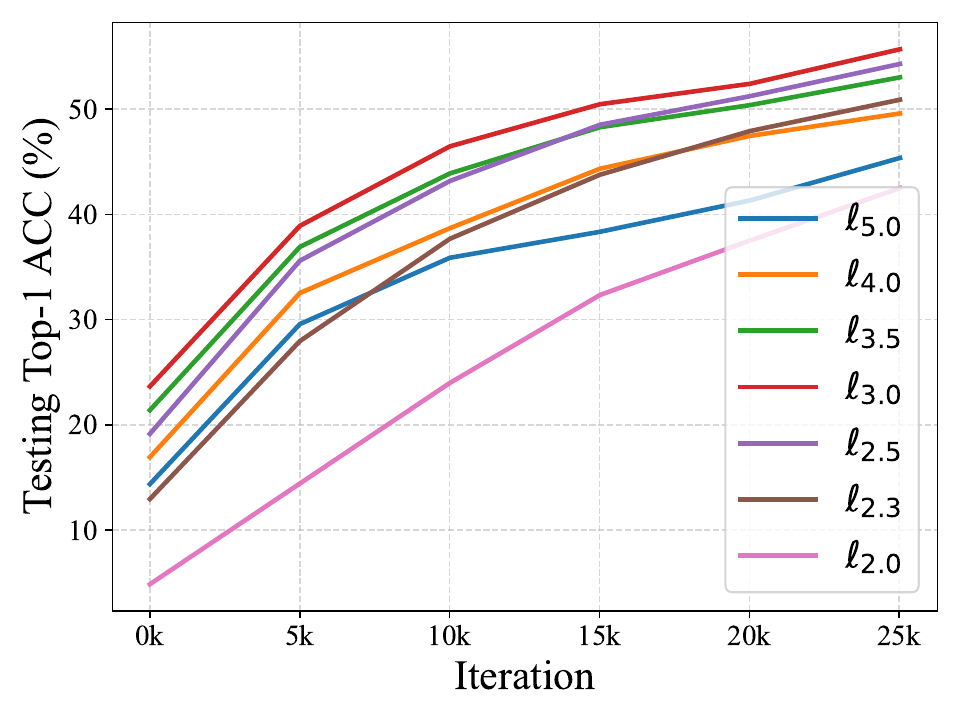}}\hspace{0pt}
    \subfloat[Training loss of {\sc Stacey}$_{(p,2)}$]{\includegraphics[width=.49\columnwidth]{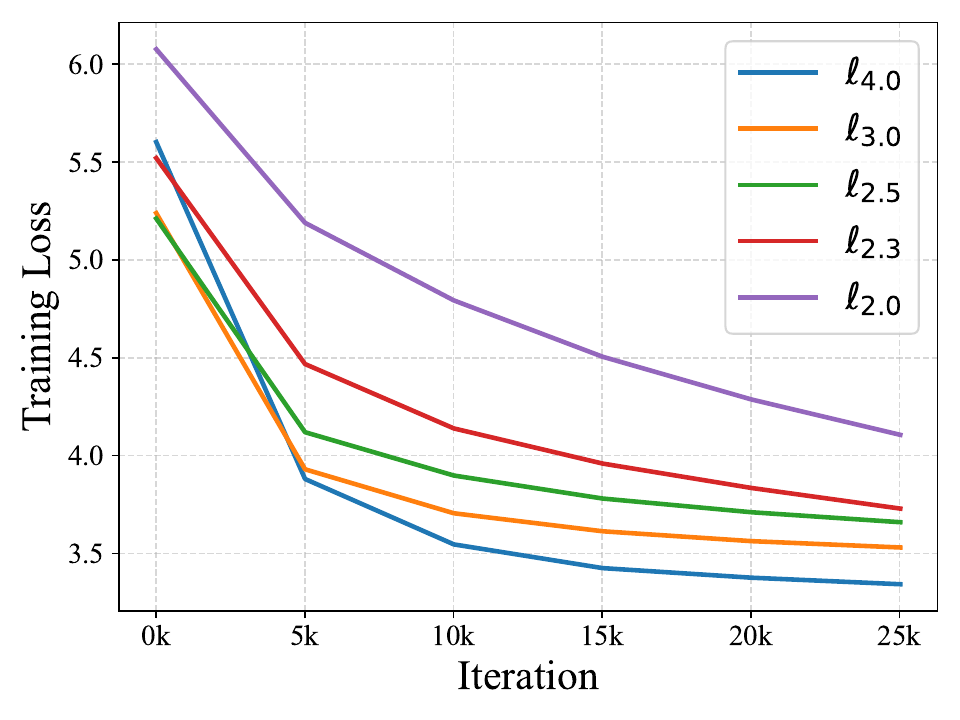}}\hspace{0pt}
    \subfloat[Testing ACC of {\sc Stacey}$_{(p,2)}$]{\includegraphics[width=.49\columnwidth]{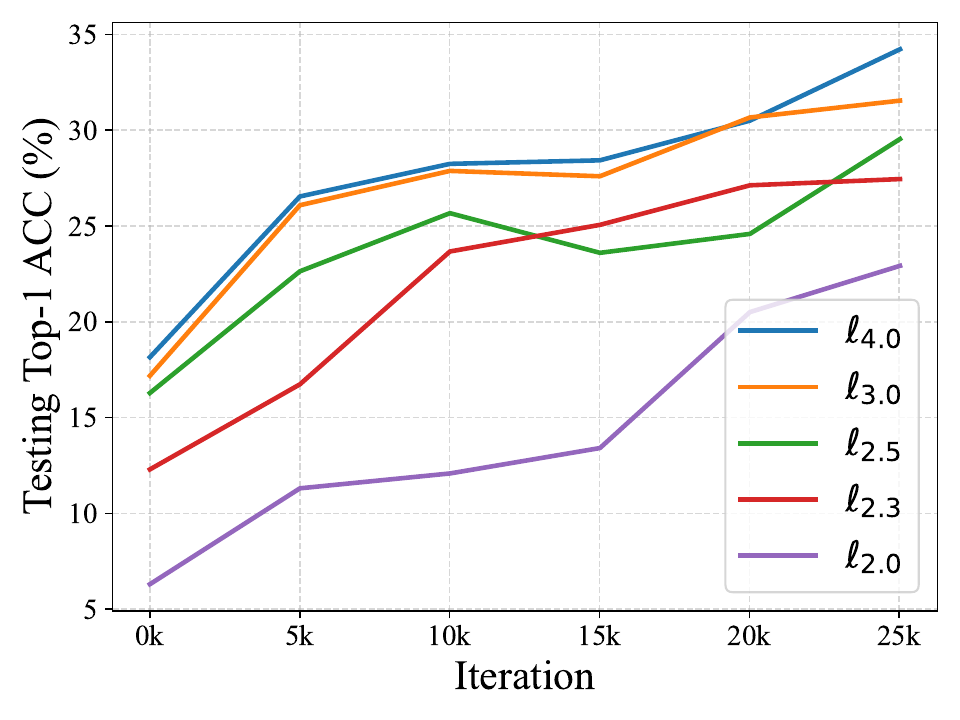}}
    \caption{Learning curves of ImageNet classification at the first 6 epochs with varying $\ell_p$-norm.}
    \label{fig:imagenet_diff_p}
\end{figure}


\begin{figure}[t]
    \centering
    \subfloat[Training loss of {\sc Stacey}$_{(p,p)}$]{\includegraphics[width=.49\columnwidth]{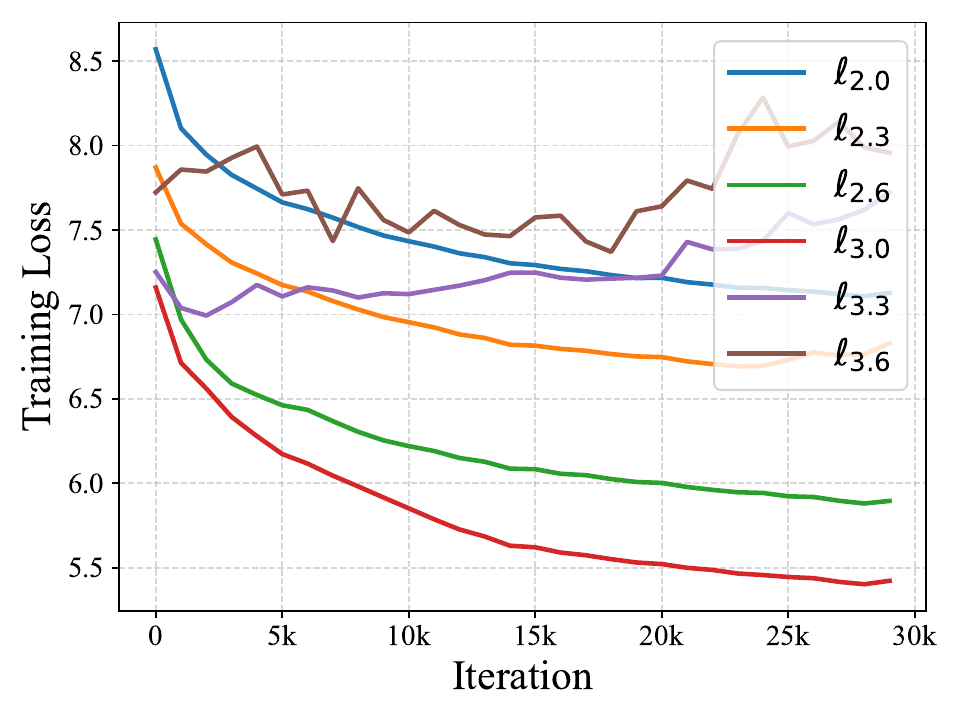}}\hspace{0pt}
    \subfloat[Testing loss of {\sc Stacey}$_{(p,p)}$]{\includegraphics[width=.49\columnwidth]{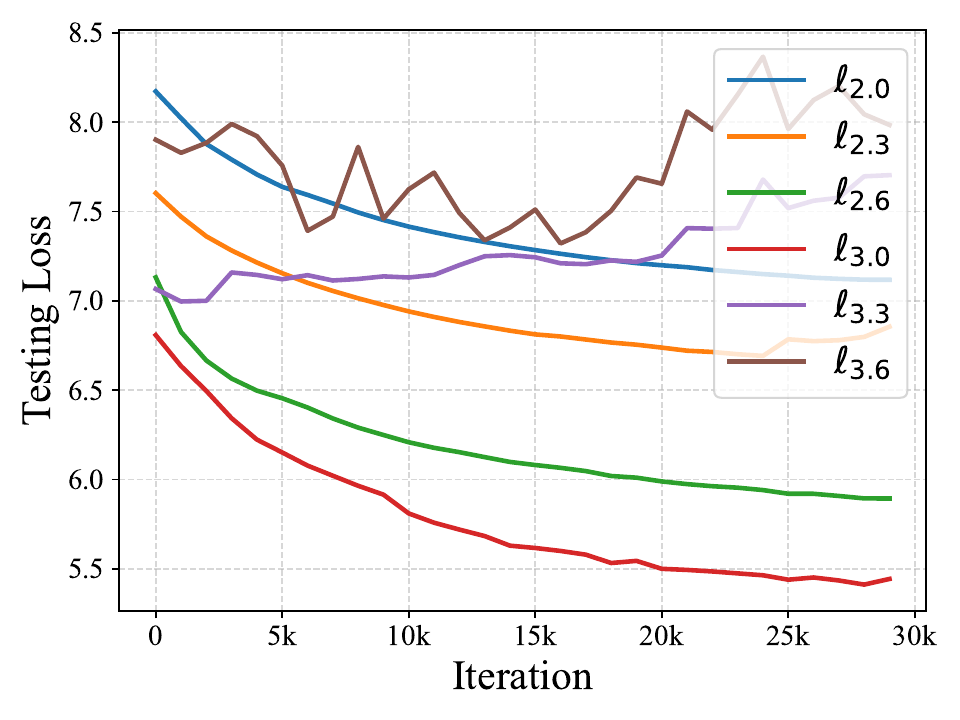}}\hspace{0pt}
    \subfloat[Training loss of {\sc Stacey}$_{(p,2)}$]{\includegraphics[width=.49\columnwidth]{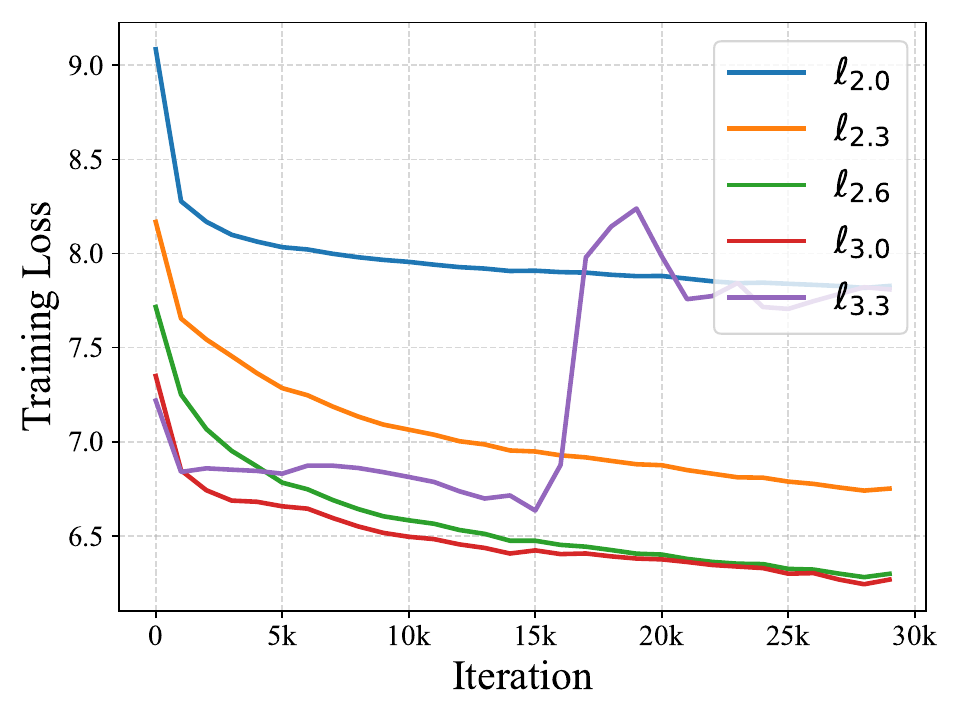}}\hspace{0pt}
    \subfloat[Testing loss of {\sc Stacey}$_{(p,2)}$]{\includegraphics[width=.49\columnwidth]{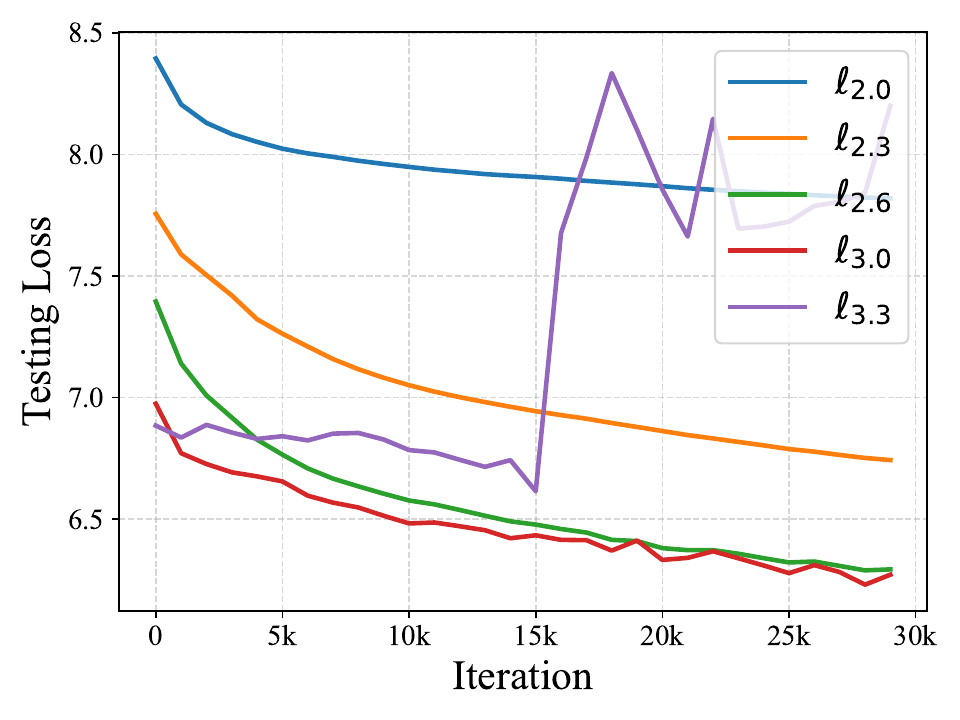}}
    \caption{Learning curves of LLM pretraining at the first 30K iterations with varying $\ell_p$-norm.}
    \label{fig:llm_diff_p}
\end{figure}

In this section, we present empirical evidence that the {\sc Stacey} optimizer outperforms other optimizers in both convergence speed and accuracy. We evaluate {\sc Stacey}’s effectiveness on image classification (Section~\ref{sec:img_classify}), and LLM pretraining (Section~\ref{sec:llm}). The hyperparameter choices and tuning are summarized in Appendix~\ref{appendix:hyper}.

In all experiments, we underscore the efficiency of the {\sc Stacey} optimizer by comparing it against other optimizers as baselines including SGD (with momentum)~\citep{nesterov1983method, polyak1964some}, Adam~\citep{kingma2014adam}, AdamW~\citep{loshchilovdecoupled}, and Lion~\citep{chen2024symbolic}.

In real-world large datasets, such as training from scratch on ImageNet~\citep{deng2009imagenet} and LLM (LLaMA~\citep{touvron2023llama}) pretraining on C4 dataset, we further demonstrate the necessity of utilizing different $\ell_p$-norms for specific tasks. For example, in the CIFAR~\citep{krizhevsky2009learning} image classification, an $\ell_p$-norm for $p$ close to $2$ delivers the best performance (Section~\ref{sec:img_classify}), consistent with the effectiveness of Euclidean-based optimizers. In contrast, an $\ell_p$-norm with $p$ around $3$ proves more effective in LLM pretraining (Section~\ref{sec:llm}). These results highlight the importance of developing non-Euclidean optimizers and adjusting the choice of $\ell_p$-norm to enhance performance across different tasks, and we would note this choice may further benefit from, e.g., parameter-free approaches~\cite{jacobsen2022parameter}.

\subsection{Image Classification}\label{sec:img_classify}
We demonstrate improved accuracy and faster convergence of the {\sc Stacey} optimizer across image classification tasks of varying scales, consistent with our algorithm’s design for acceleration. 

\paragraph{Training from scratch on CIFAR.} We train ResNet18~\citep{he2016deep} on the CIFAR dataset~\citep{krizhevsky2009learning} for 200 epochs, with the results presented in Table~\ref{tab:cifar}. We report training NLL and testing accuracy at the 50th, 100th, and 200th epochs. The proposed {\sc Stacey} optimizer consistently outperforms all compared optimizers. As shown in Fig.~\ref{fig:cifar_diff_p}, a $p$-norm of 2 yields the best performance for the CIFAR dataset when using the ResNet18 architecture.

\paragraph{Training from scratch on ImageNet.} We train ResNet50~\citep{he2016deep} with a batch size $256$ on ImageNet~\citep{deng2009imagenet} for 60 epochs.\footnote{Due to computational resource limitations, the batch sizes used in this paper are smaller than those in Lion's original paper~\citep{chen2024lion}.} The learning rate schedule is cosine decay with $10$K steps of warm-up, and the mix-precision training is used to reduce the memory footprint. The learning curves are shown in Table~\ref{tab:imagenet}.

\subsection{Pretraining Large Language Models (LLMs)}\label{sec:llm}
We pretrain \texttt{llama-100m}~\citep{touvron2023llama} on the C4 subset\footnote{\url{https://huggingface.co/datasets/datablations/c4-subsets}.} using various optimizers with cosine scheduler. The training and testing loss results, as presented in Table~\ref{tab:llm}, show the advantage of {\sc Stacey} over alternative algorithms. We additionally compare in Fig.~\ref{fig:llm_diff_p} the performance of {\sc Stacey} across different choices of $p$, whereby we observe the best performance when $p=3$, which contrasts with the best results being observed when $p=2$ in the CIFAR image classification tasks, as discussed in Section~\ref{sec:img_classify}.

\subsection{Discussion}
As we observe throughout the experiments, {\sc Stacey} demonstrates superior performance over SGD, which showcases its ability to adapt to a broader range of non-Euclidean geometries. This adaptability verifies {\sc Stacey}’s convergence for general $\ell_p$-norms, making it a better choice for optimization tasks that present complex geometries and extend beyond the conventional Euclidean frameworks.

Compared with Adam~\citep{kingma2014adam} and AdamW~\citep{loshchilovdecoupled}, the results of {\sc Stacey} suggests that the introduced acceleration technique is well-aligned with the principles of non-Euclidean optimization. In addition, they highlight how {\sc Stacey}’s acceleration mechanism, which is designed for a wider range of non-Euclidean structure, can yield better performance than traditional adaptive methods.

Furthermore, {\sc Stacey}’s improved performance over Lion~\citep{chen2024symbolic} highlights the effectiveness of interpolating primal and dual sequences as an acceleration strategy, in contrast to simply incorporating momentum. The primal-dual interpolation ensures a more balanced and stable progression towards optimality, leveraging information from both primal and dual sequences. This strategy allows {\sc Stacey} to achieve faster convergence, even in challenging settings and complex tasks like large-scale image classification and pretraining LLMs.\vspace{0.1cm}

\textbf{Algorithmic efficiency.} We observe that {\sc Stacey} has a $2d$ memory overhead, as it needs to store both a momemtum and a dual vector. This matches the memory overhead of Adam, which requires storing two moment vectors, and the per-iteration cost, in terms of basic arithmetic operations, is also comparable to that of Adam. Whereas methods such as SGD with momentum and Lion require only a single momentum vector, we would note that the overhead of the additional dual variable in {\sc Stacey} is precisely what enables its $\ell_p$-based acceleration.


\section{Conclusion}
In this paper, we have presented a new approach to stochastic non-convex optimization by leveraging \emph{non-Euclidean} $\ell_p$ geometry. We first established that \emph{stochastic $\ell_p$ steepest descent} converges at a rate of $O(\eps^{-4})$ in expectation to a stationary point under $\ell_p$-smoothness assumptions, thus strictly generalizing previous analyses for  signSGD ($p=\infty$). Building on these foundations, we introduced \textsc{Stacey}, an \emph{accelerated} algorithm that combines stochastic $\ell_p$ descent with primal-dual interpolation techniques to effectively navigate non-Euclidean optimization landscapes.

Our results highlight how acceleration in $\ell_p$ spaces can yield improved geometry-dependent performance compared to Euclidean and $\ell_\infty$-based updates. In extensive experiments on large-scale image classification and language modeling, \textsc{Stacey} consistently achieved faster convergence and higher accuracy than popular optimizers such as SGD, AdamW, and Lion. Moreover, we demonstrated the versatility of choosing different $p\in (2, \infty)$ to tailor the descent geometry to diverse model architectures and datasets. Overall, our contributions underscore both the theoretical and practical benefits of pursuing \emph{non-Euclidean} perspectives
for addressing the complexities of modern machine learning tasks.

\section*{Acknowledgements}
We thank Jincheng Zhou for helpful discussions related to the experiments implementation. Petros Drineas was partially supported by NSF AF 2209509 and NSF CDSE 2152687.

\section*{Impact Statement}
This paper presents work whose goal is to advance the field of Machine Learning. There are many potential societal consequences of our work, none of which we feel must be specifically highlighted here.





\bibliography{refs}

\begin{thebibliography}{60}
\providecommand{\natexlab}[1]{#1}
\providecommand{\url}[1]{\texttt{#1}}
\expandafter\ifx\csname urlstyle\endcsname\relax
  \providecommand{\doi}[1]{doi: #1}\else
  \providecommand{\doi}{doi: \begingroup \urlstyle{rm}\Url}\fi

\bibitem[Adil et~al.(2024)Adil, Bullins, Jambulapati, and Sidford]{adil2024convex}
Adil, D., Bullins, B., Jambulapati, A., and Sidford, A.
\newblock Convex optimization with $ p $-norm oracles.
\newblock \emph{arXiv preprint arXiv:2410.24158}, 2024.

\bibitem[Adolphs et~al.(2019)Adolphs, Kohler, and Lucchi]{adolphs2019ellipsoidal}
Adolphs, L., Kohler, J., and Lucchi, A.
\newblock Ellipsoidal trust region methods and the marginal value of hessian information for neural network training.
\newblock \emph{arXiv preprint arXiv:1905.09201}, 2019.

\bibitem[Agarwal et~al.(2009)Agarwal, Wainwright, Bartlett, and Ravikumar]{agarwal2009information}
Agarwal, A., Wainwright, M.~J., Bartlett, P., and Ravikumar, P.
\newblock Information-theoretic lower bounds on the oracle complexity of convex optimization.
\newblock \emph{Advances in Neural Information Processing Systems}, 22, 2009.

\bibitem[Allen-Zhu \& Orecchia(2017)Allen-Zhu and Orecchia]{allen2017linear}
Allen-Zhu, Z. and Orecchia, L.
\newblock Linear coupling: An ultimate unification of gradient and mirror descent.
\newblock In \emph{8th Innovations in Theoretical Computer Science Conference}. Schloss Dagstuhl--Leibniz-Zentrum f{\"u}r Informatik, 2017.

\bibitem[Arjevani et~al.(2023)Arjevani, Carmon, Duchi, Foster, Srebro, and Woodworth]{arjevani2023lower}
Arjevani, Y., Carmon, Y., Duchi, J.~C., Foster, D.~J., Srebro, N., and Woodworth, B.
\newblock Lower bounds for non-convex stochastic optimization.
\newblock \emph{Mathematical Programming}, 199\penalty0 (1):\penalty0 165--214, 2023.

\bibitem[Bai \& Bullins(2024{\natexlab{a}})Bai and Bullins]{bai2024faster}
Bai, S. and Bullins, B.
\newblock Faster acceleration for steepest descent.
\newblock \emph{arXiv preprint arXiv:2409.19200}, 2024{\natexlab{a}}.

\bibitem[Bai \& Bullins(2024{\natexlab{b}})Bai and Bullins]{bai2024local}
Bai, S. and Bullins, B.
\newblock Local composite saddle point optimization.
\newblock In \emph{International Conference on Learning Representations}, 2024{\natexlab{b}}.

\bibitem[Bai \& Bullins(2025)Bai and Bullins]{bai2025tight}
Bai, S. and Bullins, B.
\newblock Tight lower bounds under asymmetric high-order {H\"older} smoothness and uniform convexity.
\newblock In \emph{International Conference on Learning Representations}, 2025.

\bibitem[Balles et~al.(2020)Balles, Pedregosa, and Roux]{balles2020geometry}
Balles, L., Pedregosa, F., and Roux, N.~L.
\newblock The geometry of sign gradient descent.
\newblock \emph{arXiv preprint arXiv:2002.08056}, 2020.

\bibitem[Bernstein et~al.(2018)Bernstein, Wang, Azizzadenesheli, and Anandkumar]{bernstein2018signsgd}
Bernstein, J., Wang, Y.-X., Azizzadenesheli, K., and Anandkumar, A.
\newblock signsgd: Compressed optimisation for non-convex problems.
\newblock In \emph{International Conference on Machine Learning}, pp.\  560--569. PMLR, 2018.

\bibitem[Bullins(2020)]{bullins2020highly}
Bullins, B.
\newblock Highly smooth minimization of non-smooth problems.
\newblock In \emph{Conference on Learning Theory}, pp.\  988--1030. PMLR, 2020.

\bibitem[Carmon et~al.(2017)Carmon, Duchi, Hinder, and Sidford]{carmon2017convex}
Carmon, Y., Duchi, J.~C., Hinder, O., and Sidford, A.
\newblock “{Convex} until proven guilty”: dimension-free acceleration of gradient descent on non-convex functions.
\newblock In \emph{International Conference on Machine Learning}, pp.\  654--663. PMLR, 2017.

\bibitem[Chen et~al.(2024)Chen, Liu, Liang, and Liu]{chen2024lion}
Chen, L., Liu, B., Liang, K., and Liu, Q.
\newblock Lion secretly solves a constrained optimization: As lyapunov predicts.
\newblock In \emph{International Conference on Learning Representations}, 2024.

\bibitem[Chen et~al.(2023)Chen, Liang, Huang, Real, Wang, Pham, Dong, Luong, Hsieh, Lu, et~al.]{chen2024symbolic}
Chen, X., Liang, C., Huang, D., Real, E., Wang, K., Pham, H., Dong, X., Luong, T., Hsieh, C.-J., Lu, Y., et~al.
\newblock Symbolic discovery of optimization algorithms.
\newblock \emph{Advances in Neural Information Processing Systems}, 36, 2023.

\bibitem[Cohen et~al.(2021)Cohen, Kaur, Li, Kolter, and Talwalkar]{cohen2021gradient}
Cohen, J., Kaur, S., Li, Y., Kolter, J.~Z., and Talwalkar, A.
\newblock Gradient descent on neural networks typically occurs at the edge of stability.
\newblock In \emph{International Conference on Learning Representations}, 2021.

\bibitem[Contreras et~al.(2024)Contreras, Guzm{\'a}n, and Mart{\'\i}nez-Rubio]{contreras2024non}
Contreras, J.~P., Guzm{\'a}n, C., and Mart{\'\i}nez-Rubio, D.
\newblock Non-euclidean high-order smooth convex optimization.
\newblock \emph{arXiv preprint arXiv:2411.08987}, 2024.

\bibitem[Defazio et~al.(2024)Defazio, Yang, Khaled, Mishchenko, Mehta, and Cutkosky]{defazio2024road}
Defazio, A., Yang, X., Khaled, A., Mishchenko, K., Mehta, H., and Cutkosky, A.
\newblock The road less scheduled.
\newblock \emph{Advances in Neural Information Processing Systems}, 37:\penalty0 9974--10007, 2024.

\bibitem[Demidovich et~al.(2023)Demidovich, Malinovsky, Sokolov, and Richt{\'a}rik]{demidovich2023a}
Demidovich, Y., Malinovsky, G., Sokolov, I., and Richt{\'a}rik, P.
\newblock A guide through the zoo of biased {SGD}.
\newblock In \emph{Thirty-seventh Conference on Neural Information Processing Systems}, 2023.
\newblock URL \url{https://openreview.net/forum?id=OCtv4NyahI}.

\bibitem[Deng et~al.(2009)Deng, Dong, Socher, Li, Li, and Fei-Fei]{deng2009imagenet}
Deng, J., Dong, W., Socher, R., Li, L.-J., Li, K., and Fei-Fei, L.
\newblock Imagenet: A large-scale hierarchical image database.
\newblock In \emph{2009 IEEE Conference on Computer Vision and Pattern Recognition}, pp.\  248--255. Ieee, 2009.

\bibitem[Diakonikolas \& Guzm{\'a}n(2024)Diakonikolas and Guzm{\'a}n]{diakonikolas2024complementary}
Diakonikolas, J. and Guzm{\'a}n, C.
\newblock Complementary composite minimization, small gradients in general norms, and applications.
\newblock \emph{Mathematical Programming}, pp.\  1--45, 2024.

\bibitem[Diakonikolas \& Orecchia(2019)Diakonikolas and Orecchia]{diakonikolas2019approximate}
Diakonikolas, J. and Orecchia, L.
\newblock The approximate duality gap technique: A unified theory of first-order methods.
\newblock \emph{SIAM Journal on Optimization}, 29\penalty0 (1):\penalty0 660--689, 2019.

\bibitem[Duchi et~al.(2011{\natexlab{a}})Duchi, Hazan, and Singer]{duchi2011adaptive}
Duchi, J., Hazan, E., and Singer, Y.
\newblock Adaptive subgradient methods for online learning and stochastic optimization.
\newblock \emph{Journal of Machine Learning Research}, 12\penalty0 (7), 2011{\natexlab{a}}.

\bibitem[Duchi et~al.(2011{\natexlab{b}})Duchi, Agarwal, and Wainwright]{duchi2011dual}
Duchi, J.~C., Agarwal, A., and Wainwright, M.~J.
\newblock Dual averaging for distributed optimization: Convergence analysis and network scaling.
\newblock \emph{IEEE Transactions on Automatic control}, 57\penalty0 (3):\penalty0 592--606, 2011{\natexlab{b}}.

\bibitem[Ghadimi \& Lan(2013)Ghadimi and Lan]{ghadimi2013stochastic}
Ghadimi, S. and Lan, G.
\newblock Stochastic first-and zeroth-order methods for nonconvex stochastic programming.
\newblock \emph{SIAM journal on optimization}, 23\penalty0 (4):\penalty0 2341--2368, 2013.

\bibitem[Ghorbani et~al.(2019)Ghorbani, Krishnan, and Xiao]{ghorbani2019investigation}
Ghorbani, B., Krishnan, S., and Xiao, Y.
\newblock An investigation into neural net optimization via {Hessian} eigenvalue density.
\newblock In \emph{International Conference on Machine Learning}, pp.\  2232--2241. PMLR, 2019.

\bibitem[Gupta et~al.(2018)Gupta, Koren, and Singer]{gupta2018shampoo}
Gupta, V., Koren, T., and Singer, Y.
\newblock Shampoo: Preconditioned stochastic tensor optimization.
\newblock In \emph{International Conference on Machine Learning}, pp.\  1842--1850. PMLR, 2018.

\bibitem[Guzm{\'a}n \& Nemirovski(2015)Guzm{\'a}n and Nemirovski]{guzman2015lower}
Guzm{\'a}n, C. and Nemirovski, A.
\newblock On lower complexity bounds for large-scale smooth convex optimization.
\newblock \emph{Journal of Complexity}, 31\penalty0 (1):\penalty0 1--14, 2015.

\bibitem[He et~al.(2016)He, Zhang, Ren, and Sun]{he2016deep}
He, K., Zhang, X., Ren, S., and Sun, J.
\newblock Deep residual learning for image recognition.
\newblock In \emph{Proceedings of the IEEE Conference on Computer Vision and Pattern Recognition}, pp.\  770--778, 2016.

\bibitem[Jacobsen \& Cutkosky(2022)Jacobsen and Cutkosky]{jacobsen2022parameter}
Jacobsen, A. and Cutkosky, A.
\newblock Parameter-free mirror descent.
\newblock In \emph{Conference on Learning Theory}, pp.\  4160--4211. PMLR, 2022.

\bibitem[Jambulapati et~al.(2019)Jambulapati, Sidford, and Tian]{jambulapati2019direct}
Jambulapati, A., Sidford, A., and Tian, K.
\newblock A direct $\tilde{O}(1/\epsilon)$ iteration parallel algorithm for optimal transport.
\newblock \emph{Advances in Neural Information Processing Systems}, 32, 2019.

\bibitem[Jiang et~al.(2024)Jiang, Malik, and Li]{jiang2024does}
Jiang, K., Malik, D., and Li, Y.
\newblock How does adaptive optimization impact local neural network geometry?
\newblock \emph{Advances in Neural Information Processing Systems}, 36, 2024.

\bibitem[Jin et~al.(2017)Jin, Ge, Netrapalli, Kakade, and Jordan]{jin2017escape}
Jin, C., Ge, R., Netrapalli, P., Kakade, S.~M., and Jordan, M.~I.
\newblock How to escape saddle points efficiently.
\newblock In \emph{International Conference on Machine Learning}, pp.\  1724--1732. PMLR, 2017.

\bibitem[Karimi et~al.(2016)Karimi, Nutini, and Schmidt]{karimi2016linear}
Karimi, H., Nutini, J., and Schmidt, M.
\newblock Linear convergence of gradient and proximal-gradient methods under the {Polyak-{\L}ojasiewicz} condition.
\newblock In \emph{Machine Learning and Knowledge Discovery in Databases: European Conference, ECML PKDD 2016, Riva del Garda, Italy, September 19-23, 2016, Proceedings, Part I 16}, pp.\  795--811. Springer, 2016.

\bibitem[Kelner et~al.(2014)Kelner, Lee, Orecchia, and Sidford]{kelner2014almost}
Kelner, J.~A., Lee, Y.~T., Orecchia, L., and Sidford, A.
\newblock An almost-linear-time algorithm for approximate max flow in undirected graphs, and its multicommodity generalizations.
\newblock In \emph{Proceedings of the Twenty-Fifth Annual ACM-SIAM Symposium on Discrete Algorithms}, pp.\  217--226. SIAM, 2014.

\bibitem[Kingma \& Ba(2015)Kingma and Ba]{kingma2014adam}
Kingma, D.~P. and Ba, J.
\newblock {Adam}: A method for stochastic optimization.
\newblock In \emph{International Conference on Learning Representations}, 2015.

\bibitem[Krizhevsky(2009)]{krizhevsky2009learning}
Krizhevsky, A.
\newblock Learning multiple layers of features from tiny images.
\newblock \emph{Master's thesis, University of Tront}, 2009.

\bibitem[Li et~al.(2020{\natexlab{a}})Li, Gu, Zhou, Chen, and Banerjee]{li2020hessian}
Li, X., Gu, Q., Zhou, Y., Chen, T., and Banerjee, A.
\newblock Hessian based analysis of sgd for deep nets: Dynamics and generalization.
\newblock In \emph{Proceedings of the 2020 SIAM International Conference on Data Mining}, pp.\  190--198. SIAM, 2020{\natexlab{a}}.

\bibitem[Li et~al.(2020{\natexlab{b}})Li, Huang, Yang, Wang, and Zhang]{li2020On}
Li, X., Huang, K., Yang, W., Wang, S., and Zhang, Z.
\newblock On the convergence of fedavg on non-iid data.
\newblock In \emph{International Conference on Learning Representations}, 2020{\natexlab{b}}.

\bibitem[Liu et~al.(2024)Liu, Li, Hall, Liang, and Ma]{liu2024sophia}
Liu, H., Li, Z., Hall, D. L.~W., Liang, P., and Ma, T.
\newblock Sophia: A scalable stochastic second-order optimizer for language model pre-training.
\newblock In \emph{International Conference on Learning Representations}, 2024.

\bibitem[Loshchilov \& Hutter(2019)Loshchilov and Hutter]{loshchilovdecoupled}
Loshchilov, I. and Hutter, F.
\newblock Decoupled weight decay regularization.
\newblock In \emph{International Conference on Learning Representations}, 2019.

\bibitem[Martens \& Grosse(2015)Martens and Grosse]{martens2015optimizing}
Martens, J. and Grosse, R.
\newblock Optimizing neural networks with kronecker-factored approximate curvature.
\newblock In \emph{International conference on machine learning}, pp.\  2408--2417. PMLR, 2015.

\bibitem[Morwani et~al.(2025)Morwani, Shapira, Vyas, eran malach, Kakade, and Janson]{morwani2025a}
Morwani, D., Shapira, I., Vyas, N., eran malach, Kakade, S.~M., and Janson, L.
\newblock A new perspective on shampoo's preconditioner.
\newblock In \emph{The Thirteenth International Conference on Learning Representations}, 2025.
\newblock URL \url{https://openreview.net/forum?id=c6zI3Cp8c6}.

\bibitem[Nemirovskii \& Nesterov(1985)Nemirovskii and Nesterov]{nemirovskii1985optimal}
Nemirovskii, A.~S. and Nesterov, Y.~E.
\newblock Optimal methods of smooth convex minimization.
\newblock \emph{USSR Computational Mathematics and Mathematical Physics}, 25\penalty0 (2):\penalty0 21--30, 1985.

\bibitem[Nemirovskij \& Yudin(1983)Nemirovskij and Yudin]{nemirovskij1983problem}
Nemirovskij, A.~S. and Yudin, D.~B.
\newblock \emph{Problem Complexity and Method Efficiency in Optimization}.
\newblock A Wiley-Interscience publication. Wiley, 1983.
\newblock ISBN 9780471103455.

\bibitem[Nesterov(1983)]{nesterov1983method}
Nesterov, Y.
\newblock A method for solving the convex programming problem with convergence rate {O}(1/k2).
\newblock In \emph{Dokl. Akad. Nauk SSSR}, volume 269, pp.\  543, 1983.

\bibitem[Nesterov(2005)]{nesterov2005smooth}
Nesterov, Y.
\newblock Smooth minimization of non-smooth functions.
\newblock \emph{Mathematical programming}, 103:\penalty0 127--152, 2005.

\bibitem[Nesterov(2018)]{nesterov2018lectures}
Nesterov, Y.
\newblock \emph{Lectures on convex optimization}, volume 137.
\newblock Springer, 2018.

\bibitem[Papyan(2018)]{papyan2018full}
Papyan, V.
\newblock The full spectrum of deepnet {Hessians} at scale: Dynamics with sgd training and sample size.
\newblock \emph{arXiv preprint arXiv:1811.07062}, 2018.

\bibitem[Polyak(1964)]{polyak1964some}
Polyak, B.~T.
\newblock Some methods of speeding up the convergence of iteration methods.
\newblock \emph{USSR Computational Mathematics and Mathematical Physics}, 4\penalty0 (5):\penalty0 1--17, 1964.

\bibitem[Robbins \& Monro(1951)Robbins and Monro]{robbins1951stochastic}
Robbins, H. and Monro, S.
\newblock A stochastic approximation method.
\newblock \emph{The Annals of Mathematical Statistics}, pp.\  400--407, 1951.

\bibitem[Sherman(2017)]{sherman2017area}
Sherman, J.
\newblock Area-convexity, $\ell_\infty$ regularization, and undirected multicommodity flow.
\newblock In \emph{Proceedings of the 49th Annual ACM SIGACT Symposium on Theory of Computing}, pp.\  452--460, 2017.

\bibitem[Sidford \& Tian(2018)Sidford and Tian]{sidford2018coordinate}
Sidford, A. and Tian, K.
\newblock Coordinate methods for accelerating $\ell_\infty$ regression and faster approximate maximum flow.
\newblock In \emph{2018 IEEE 59th Annual Symposium on Foundations of Computer Science}, pp.\  922--933. IEEE, 2018.

\bibitem[Song et~al.(2021)Song, Jiang, and Ma]{song2021unified}
Song, C., Jiang, Y., and Ma, Y.
\newblock Unified acceleration of high-order algorithms under general holder continuity.
\newblock \emph{SIAM Journal on Optimization}, 31\penalty0 (3):\penalty0 1797--1826, 2021.

\bibitem[Stich \& Ajalloeian(2020)Stich and Ajalloeian]{stich2020analysis}
Stich, S.~U. and Ajalloeian, A.
\newblock Analysis of sgd with biased gradient estimators.
\newblock \emph{arXiv preprint arXiv:2008.00051}, 2020.

\bibitem[Sutskever et~al.(2013)Sutskever, Martens, Dahl, and Hinton]{sutskever2013importance}
Sutskever, I., Martens, J., Dahl, G., and Hinton, G.
\newblock On the importance of initialization and momentum in deep learning.
\newblock In \emph{International Conference on Machine Learning}, pp.\  1139--1147. PMLR, 2013.

\bibitem[Touvron et~al.(2023)Touvron, Lavril, Izacard, Martinet, Lachaux, Lacroix, Rozi{\`e}re, Goyal, Hambro, Azhar, et~al.]{touvron2023llama}
Touvron, H., Lavril, T., Izacard, G., Martinet, X., Lachaux, M.-A., Lacroix, T., Rozi{\`e}re, B., Goyal, N., Hambro, E., Azhar, F., et~al.
\newblock Llama: Open and efficient foundation language models.
\newblock \emph{arXiv preprint arXiv:2302.13971}, 2023.

\bibitem[Vyas et~al.(2025)Vyas, Morwani, Zhao, Shapira, Brandfonbrener, Janson, and Kakade]{vyas2025soap}
Vyas, N., Morwani, D., Zhao, R., Shapira, I., Brandfonbrener, D., Janson, L., and Kakade, S.~M.
\newblock {SOAP}: Improving and stabilizing shampoo using adam for language modeling.
\newblock In \emph{The Thirteenth International Conference on Learning Representations}, 2025.
\newblock URL \url{https://openreview.net/forum?id=IDxZhXrpNf}.

\bibitem[Wang et~al.(2017)Wang, Fang, and Liu]{wang2017stochastic}
Wang, M., Fang, E.~X., and Liu, H.
\newblock Stochastic compositional gradient descent: algorithms for minimizing compositions of expected-value functions.
\newblock \emph{Mathematical Programming}, 161:\penalty0 419--449, 2017.

\bibitem[Wang et~al.(2024)Wang, Chen, Jiang, Yang, Wan, and Zhang]{wang2024online}
Wang, Y., Chen, S., Jiang, W., Yang, W., Wan, Y., and Zhang, L.
\newblock Online composite optimization between stochastic and adversarial environments.
\newblock In \emph{The Thirty-eighth Annual Conference on Neural Information Processing Systems}, 2024.
\newblock URL \url{https://openreview.net/forum?id=MbEB5aKmMK}.

\bibitem[Yuan et~al.(2021)Yuan, Zaheer, and Reddi]{yuan2021federated}
Yuan, H., Zaheer, M., and Reddi, S.
\newblock Federated composite optimization.
\newblock In \emph{International Conference on Machine Learning}, pp.\  12253--12266. PMLR, 2021.

\end{thebibliography}
\bibliographystyle{icml2025}

\newpage
\appendix
\onecolumn
\section{Proofs}\label{appendix:proof}
\subsection{Complete Proof for Theorem~\ref{thm:main}} \label{app:proof}
\noindent\textbf{Theorem \ref{thm:main}} \textit{ Running Algorithm \ref{alg:ellp-descent} on some (possibly non-convex) function $f$ that satisfies Assumptions \ref{asm:smooth-lp} to \ref{asm:bound-grad} yields 
\begin{align*}
        \E\bracks{\frac{1}{T}\sum_{t=0}^{T-1} \norm{g_t}_{p^\ast}^{p^\ast}} \leq \frac{f_0 - f^\ast}{\eta T} + \frac{L\eta G^\frac{2}{p-1}}{2} + \frac{1}{T}\sum_{t=0}^{T-1}\frac{\frac{2p-1}{p-1}G^\frac{1}{p-1} \norm{\vec{\sigma}}_1}{\sqrt{n_t}} 
    \end{align*}   
where $f_0 = f(\theta_0)$ and $f^\ast = f(\theta^\ast)$, $n_t$ is the batch size in iteration $t$ and $L$, $\vec{\sigma}$, and $G$ are constants from Assumption \ref{asm:smooth-lp}, \ref{asm:coord-var}, \ref{asm:bound-grad}. Further letting the batch size $n_t = T$, the number of gradient call is $N=T^2$ for $T$ iterations. With $\eta = \frac{1}{L^\frac{1}{2}G^\frac{1}{p-1}T^\frac{1}{2}}$ we have 
\begin{align*}
        \E\bracks{\frac{1}{T}\sum_{t=0}^{T-1} \norm{g_t}_{p^\ast}^{p^\ast}} \leq \frac{1}{N^\frac{1}{4}}\bracks{L^\frac{1}{2}G^\frac{1}{p-1}\pars{f_0 - f^\ast + \frac{1}{2}} + \frac{2p-1}{p-1}G^\frac{1}{p-1} \norm{\vec{\sigma}}_1},
    \end{align*} 
   i.e., Algorithm \ref{alg:ellp-descent} takes $N \in \mc{O}\pars{\epsilon^{-4}}$ gradient queries to reach an $\epsilon$-approximate stationary point.
}
    \begin{proof}
     Starting with Assumption \ref{asm:smooth-lp} and the descent step in Algorithm \ref{alg:ellp-descent},
    \begin{align*}
        f(\theta_{t+1}) &\leq f(\theta_t) + \inner{g_t}{\theta_{t+1} - \theta_t} + \frac{L}{2} \norm{\theta_{t+1} - \theta_t}_p^2 \\
        &= f(\theta_t) + \eta\inner{g_t}{-s(\tilde{g}_t)} + \frac{L}{2} \norm{s(\tilde{g}_t)}_p^2 \\
        &= f(\theta_t) - \underbrace{\eta\inner{g_t}{s(g_t)}}_{A} + \underbrace{\eta \inner{g_t}{s(g_t) - s(\tilde{g}_t) }}_{B} + \underbrace{\frac{L\eta^2}{2} \norm{s(\tilde{g}_t)}_p^2}_{C}
    \end{align*}
Now we analyze these terms one by one.
    \begin{align*}
        A &= \sum_{i=1}^d g_t^{(i)} \cdot \frac{g_t^{(i)}}{|g_t^{(i)}|^{\frac{p-2}{p-1}}} \\
        &= \sum_{i=1}^d |g_t^{(i)}|^{\frac{p}{p-1}} \\
        &= \norm{g_t}_{p^\ast}^{p^\ast}
    \end{align*}
For term $B$,
    \begin{align*}
        B &= \eta\sum_{i=1}^d g_t^{(i)} \left(\frac{g_t^{(i)}}{|g_t^{(i)}|^{\frac{p-2}{p-1}}} - \frac{\tilde{g}_t^{(i)}}{|\tilde{g}_t^{(i)}|^{\frac{p-2}{p-1}}} \right) \\
        &= \eta\sum_{i=1}^d g_t^{(i)} \left(\sgn\left(g_t^{(i)}\right)|g_t^{(i)}|^{\frac{1}{p-1}} - \sgn\left(\tilde{g}_t^{(i)}\right)|\tilde{g}_t^{(i)}|^{\frac{1}{p-1}} \right) \\ 
        &\leq \eta\sum_{i=1}^d \abs{g_t^{(i)}} \left|\sgn\left(g_t^{(i)}\right)|g_t^{(i)}|^{\frac{1}{p-1}} - \sgn\left(\tilde{g}_t^{(i)}\right)|\tilde{g}_t^{(i)}|^{\frac{1}{p-1}}\right| \\ 
        &= \underbrace{\eta\sum_{i=1}^d \abs{g_t^{(i)}} \left(|g_t^{(i)}|^{\frac{1}{p-1}} + |\tilde{g}_t^{(i)}|^{\frac{1}{p-1}}\right) \mathbb{I}_{\bracks{\sgn\left(g_t^{(i)}\right) \neq \sgn\left(\tilde{g}_t^{(i)}\right)}}}_{B_1} \\
        & \qquad + \underbrace{\eta\sum_{i=1}^d \abs{g_t^{(i)}} \left||g_t^{(i)}|^{\frac{1}{p-1}} - |\tilde{g}_t^{(i)}|^{\frac{1}{p-1}}\right| \mathbb{I}_{\bracks{\sgn\left(g_t^{(i)}\right) = \sgn\left(\tilde{g}_t^{(i)}\right)}}}_{B_2} 
    \end{align*}
    $B_1$ is bounded in expectation by $\frac{ 2 \eta G^\frac{1}{p-1} \norm{\vec{\sigma}}_1}{\sqrt{n_t}}$ in Lemma \ref{lem:B1} and $B_2$ is bounded in expectation by $\frac{\eta G^\frac{1}{p-1} \norm{\vec{\sigma}}_1}{(p-1)\sqrt{n_t}}$ in Lemma \ref{lem:B2}.
    \begin{align*}
        C &= \frac{L\eta^2}{2} \left(\sum_{i=1}^d \left |\frac{g_t^{(i)}}{|g_t^{(i)}|^{\frac{p-2}{p-1}}} \right |^p\right)^\frac{2}{p} \\
        &= \frac{L\eta^2}{2} \left(\sum_{i=1}^d \left |g_t^{(i)}\right |^\frac{p}{p-1}\right)^\frac{2}{p} \\
        &= \frac{L\eta^2}{2} \norm{g_t}_{p^\ast}^{\frac{2}{p-1}} \\
        &\leq \frac{L\eta^2G^\frac{2}{p-1}}{2}
    \end{align*}
    Therefore,
    \begin{align*}
        \eta \E\bracks{\norm{g_t}_{p^\ast}^{p^\ast}} \leq f(\theta_t) - f(\theta_{t+1}) + \frac{\eta (2p-1)G^\frac{1}{p-1} \norm{\vec{\sigma}}_1}{(p-1)\sqrt{n_t}} + \frac{L\eta^2G^\frac{2}{p-1}}{2}
    \end{align*}

    By telescoping through $t = 0, \cdots, T-1$, we get
    \begin{align*}
        \E\bracks{\frac{1}{T}\sum_{t=0}^{T-1} \norm{g_t}_{p^\ast}^{p^\ast}} \leq \frac{f(\theta_0) - f(\theta_T)}{\eta T} + \frac{1}{T}\sum_{t=0}^{T-1}\frac{(2p-1)G^\frac{1}{p-1} \norm{\vec{\sigma}}_1}{(p-1)\sqrt{n_t}} + \frac{L\eta G^\frac{2}{p-1}}{2}
    \end{align*}
\end{proof}

\begin{lemma} \label{lem:B1}
    $$\E\bracks{\eta\sum_{i=1}^d \abs{g_t^{(i)}} \left(|g_t^{(i)}|^{\frac{1}{p-1}} + |\tilde{g}_t^{(i)}|^{\frac{1}{p-1}}\right)\mathbb{I}_{\bracks{\sgn\left(g_t^{(i)}\right) \neq \sgn\left(\tilde{g}_t^{(i)}\right)}}} \leq \frac{ 2 \eta G^\frac{1}{p-1} \norm{\vec{\sigma}}_1}{\sqrt{n_t}}$$
\end{lemma}
\begin{proof} By Corollary \ref{cor:bound-grad} (b),
    \begin{align*}
        \E&\bracks{\eta \sum_{i=1}^d \abs{g_t^{(i)}} \left(|g_t^{(i)}|^{\frac{1}{p-1}} + |\tilde{g}_t^{(i)}|^{\frac{1}{p-1}}\right) \mathbb{I}_{\bracks{\sgn\left(g_t^{(i)}\right) \neq \sgn\left(\tilde{g}_t^{(i)}\right)}}} \\
        &\leq 2\eta G^\frac{1}{p-1} \E\bracks{\sum_{i=1}^d \abs{g_t^{(i)}}\mathbb{I}_{\bracks{\sgn\left(g_t^{(i)}\right) \neq \sgn\left(\tilde{g}_t^{(i)}\right)}}} \\
        &= 2\eta G^\frac{1}{p-1} \sum_{i=1}^d \abs{g_t^{(i)}}\mathbb{P}\bracks{\sgn\left(g_t^{(i)}\right) \neq \sgn\left(\tilde{g}_t^{(i)}\right)} \\
        &\leq 2\eta G^\frac{1}{p-1} \sum_{i=1}^d \abs{g_t^{(i)}}\mathbb{P}\bracks{\abs{\tilde{g}_t^{(i)} - g_t^{(i)}} \geq \abs{g_t^{(i)}} } \\
        &\leq 2\eta G^\frac{1}{p-1} \sum_{i=1}^d \abs{g_t^{(i)}}\frac{\mathbb{E}\bracks{\abs{\tilde{g}_t^{(i)} - g_t^{(i)}}}}{\abs{g_t^{(i)}}} \\
        &\leq 2\eta G^\frac{1}{p-1} \sum_{i=1}^d \sqrt{\mathbb{E}\bracks{\abs{\tilde{g}_t^{(i)} - g_t^{(i)}}^2}} \\
        &\leq \frac{2\eta G^\frac{1}{p-1} \sum_{i=1}^d \sigma_i}{\sqrt{n_t}}  \\
        &= \frac{2\eta G^\frac{1}{p-1} \norm{\vec{\sigma}}_1}{\sqrt{n_t}} 
    \end{align*}
    where for the last three inequalities we used Markov's inequality, Jensen's inequality, and Assumption \ref{asm:coord-var}.
\end{proof}

\begin{lemma} \label{lem:B2} 
    $\E\bracks{\eta\sum_{i=1}^d \abs{g_t^{(i)}} \left||g_t^{(i)}|^{\frac{1}{p-1}} - |\tilde{g}_t^{(i)}|^{\frac{1}{p-1}}\right| \mathbb{I}_{\bracks{\sgn\left(g_t^{(i)}\right) = \sgn\left(\tilde{g}_t^{(i)}\right)}}} \leq \frac{\eta G^\frac{1}{p-1} \norm{\vec{\sigma}}_1}{(p-1)\sqrt{n_t}}.$
\end{lemma} \vspace{-10pt}
\begin{proof}
Denoting $\E\bracks{\cdot \mid \sgn\left(g_t^{(i)}\right) = \sgn\left(\tilde{g}_t^{(i)}\right)}$ as $\E_{\mid =}\bracks{\cdot}$, and $\mathbb{P}\bracks{\sgn\left(g_t^{(i)}\right) = \sgn\left(\tilde{g}_t^{(i)}\right)}$ as $\mathbb{P}\bracks{=}$,
    \begin{align*}
    \E&\bracks{\eta\sum_{i=1}^d \abs{g_t^{(i)}} \left||g_t^{(i)}|^{\frac{1}{p-1}} - |\tilde{g}_t^{(i)}|^{\frac{1}{p-1}}\right| \mathbb{I}_{\bracks{\sgn\left(g_t^{(i)}\right) = \sgn\left(\tilde{g}_t^{(i)}\right)}}} \\
    &= \eta\E_{\mid =}\bracks{\sum_{i=1}^d \abs{g_t^{(i)}} \left||g_t^{(i)}|^{\frac{1}{p-1}} - |\tilde{g}_t^{(i)}|^{\frac{1}{p-1}}\right|} \mathbb{P}\bracks{=} \\
    &= \eta\E_{\mid =}\bracks{\sum_{i=1}^d \abs{g_t^{(i)}} \left||g_t^{(i)}|^{\frac{1}{p-1}} - |\tilde{g}_t^{(i)}|^{\frac{1}{p-1}}\right|}\mathbb{P}\bracks{=} \\
    &= \eta\E_{\mid =}\bracks{\sum_{i=1}^d \abs{g_t^{(i)}} \left|\pars{|\tilde{g}_t^{(i)}|^{\frac{1}{p-1}} + \frac{1}{p-1}\sgn(\zeta^{(i)})\abs{\zeta^{(i)}}^\frac{2-p}{p-1} \pars{g_t^{(i)} - \tilde{g}_t^{(i)}}} - |\tilde{g}_t^{(i)}|^{\frac{1}{p-1}}\right|}\mathbb{P}\bracks{=} \\
    &= \eta\E_{\mid =}\bracks{\sum_{i=1}^d \abs{g_t^{(i)}} \left|\frac{1}{p-1}\sgn(\zeta^{(i)})\abs{\zeta^{(i)}}^\frac{2-p}{p-1} \pars{g_t^{(i)} - \tilde{g}_t^{(i)}} \right|} \mathbb{P}\bracks{=}\\
    &= \frac{\eta}{p-1}\E_{\mid =}\bracks{\sum_{i=1}^d \abs{g_t^{(i)}} \abs{\zeta^{(i)}}^\frac{2-p}{p-1} \abs{g_t^{(i)} - \tilde{g}_t^{(i)}} }\mathbb{P}\bracks{=}, \vspace{-5pt}
\end{align*}
in which the second equality holds by taking the zeroth order Taylor expansion of $\abs{g_t^{(i)}}^{\frac{1}{p-1}}$ at $\tilde{g}_t^{(i)}$ with Lagrange remainder, and $\zeta^{(i)}$ is in the range from $g_t^{(i)}$ to $\tilde{g}_t^{(i)}$. Given $\sgn\left(g_t^{(i)}\right) = \sgn\left(\tilde{g}_t^{(i)}\right)$, by the definition of $\zeta^{(i)}$ in the Lagrange remainder, we must have either $\abs{g_t^{(i)}} \leq \abs{\zeta^{(i)}} \leq \abs{\tilde{g}_t^{(i)}}$ 
 or $\abs{g_t^{(i)}} \geq \abs{\zeta^{(i)}} \geq \abs{\tilde{g}_t^{(i)}}$. Now we analyze these two cases respectively. We write out the derivations separately for clarity and simplicity, alternatively one can merge these two cases with the law of total expectation.

 (1) If $\abs{g_t^{(i)}} \leq \abs{\zeta^{(i)}} \leq \abs{\tilde{g}_t^{(i)}}$, then \vspace{-5pt}
 
\resizebox{\textwidth}{!}{
$\begin{aligned}
    \frac{\eta}{p-1}\E_{\mid =}\bracks{\sum_{i=1}^d \abs{g_t^{(i)}} \abs{\zeta^{(i)}}^\frac{2-p}{p-1} \abs{g_t^{(i)} - \tilde{g}_t^{(i)}} }\mathbb{P}\bracks{=} &\leq \frac{\eta}{p-1}\E_{\mid =}\bracks{\sum_{i=1}^d \abs{\zeta^{(i)}} \abs{\zeta^{(i)}}^\frac{2-p}{p-1} \abs{g_t^{(i)} - \tilde{g}_t^{(i)}} }\mathbb{P}\bracks{=}\\
     &=\frac{\eta}{p-1}\E_{\mid =}\bracks{\sum_{i=1}^d  \abs{\zeta^{(i)}}^\frac{1}{p-1} \abs{g_t^{(i)} - \tilde{g}_t^{(i)}} } \mathbb{P}\bracks{=}\\
    &\leq \frac{\eta}{p-1}\E_{\mid =}\bracks{\sum_{i=1}^d  \abs{\tilde{g}_t^{(i)}}^\frac{1}{p-1} \abs{g_t^{(i)} - \tilde{g}_t^{(i)}} } \mathbb{P}\bracks{=}\\
    & \leq \frac{\eta G^\frac{1}{p-1}}{p-1}\sum_{i=1}^d \E_{\mid =}\bracks{ \abs{g_t^{(i)} - \tilde{g}_t^{(i)}} }\mathbb{P}\bracks{=} \\
    & = \frac{\eta G^\frac{1}{p-1}}{p-1}\sum_{i=1}^d\frac{\E\bracks{ \abs{g_t^{(i)} - \tilde{g}_t^{(i)}} }}{\mathbb{P}[=]} \mathbb{P}\bracks{=}\\
    &\leq \frac{\eta G^\frac{1}{p-1}}{p-1}\sum_{i=1}^d \sqrt{\E\bracks{ \abs{g_t^{(i)} - \tilde{g}_t^{(i)}}^2 } } & \text{(Jensen's)}\\
    &\leq \frac{\eta G^\frac{1}{p-1}}{p-1}\sum_{i=1}^d \frac{\sigma_i}{\sqrt{n_t}} & \text{(Assumption \ref{asm:coord-var})} \\
    &= \frac{\eta G^\frac{1}{p-1} \norm{\vec{\sigma}}_1}{(p-1)\sqrt{n_t}}
\end{aligned}$}

(2) If $\abs{g_t^{(i)}} \geq \abs{\zeta^{(i)}} \geq \abs{\tilde{g}_t^{(i)}}$, then
\begin{align*}
&\frac{\eta}{p-1}\E_{\mid =}\bracks{\sum_{i=1}^d \abs{g_t^{(i)}} \abs{\zeta^{(i)}}^\frac{2-p}{p-1} \abs{g_t^{(i)} - \tilde{g}_t^{(i)}} }\mathbb{P}\bracks{=} \\
&\leq \frac{\eta}{p-1}\E_{\mid =}\bracks{\sum_{i=1}^d \abs{g_t^{(i)}} \abs{\tilde{g}_t^{(i)}}^\frac{2-p}{p-1} \abs{g_t^{(i)} - \tilde{g}_t^{(i)}} }\mathbb{P}\bracks{=} \\
&\leq \frac{\eta}{(p-1)\mathbb{P}[=]}\E\bracks{\sum_{i=1}^d \abs{g_t^{(i)}} \abs{\tilde{g}_t^{(i)}}^\frac{2-p}{p-1} \abs{g_t^{(i)} - \tilde{g}_t^{(i)}} }\mathbb{P}\bracks{=} \\
    &\leq \frac{\eta}{p-1}\sum_{i=1}^d 
 \sqrt{\E\bracks{\abs{g_t^{(i)}}^2 \abs{\tilde{g}_t^{(i)}}^\frac{2(2-p)}{p-1}}\E\bracks{\abs{g_t^{(i)} - \tilde{g}_t^{(i)}}^2 }} & \text{(Cauchy-Schwarz)} \\
 &\leq \frac{\eta}{p-1}\sum_{i=1}^d 
 \sqrt{\abs{g_t^{(i)}}^2 \E\bracks{\abs{\tilde{g}_t^{(i)}}^\frac{2(2-p)}{p-1}}\frac{\sigma_i^2}{n_t}} & \text{(Assumption \ref{asm:coord-var})}\\
 &\leq \frac{\eta}{p-1}\sum_{i=1}^d 
 \sqrt{\abs{g_t^{(i)}}^2 \pars{\E\bracks{\abs{\tilde{g}_t^{(i)}}^2}}^\frac{2-p}{p-1}\frac{\sigma_i^2}{n_t}} & \text{(Jensen's)}\\
 &\leq \frac{\eta}{p-1}\sum_{i=1}^d 
 \sqrt{\abs{g_t^{(i)}}^2 \pars{\Var\bracks{\tilde{g}_t^{(i)}} + \pars{\E\bracks{\tilde{g}_t^{(i)}}}^2}^\frac{2-p}{p-1}\frac{\sigma_i^2}{n_t}} & \text{(Variance Definition)}\\
 &\leq \frac{\eta}{p-1}\sum_{i=1}^d 
 \sqrt{\abs{g_t^{(i)}}^2 \pars{\E\bracks{\tilde{g}_t^{(i)}}}^\frac{2(2-p)}{p-1}\frac{\sigma_i^2}{n_t}}\\
 &= \frac{\eta}{p-1}\sum_{i=1}^d 
 \abs{g_t^{(i)}}^\frac{1}{p-1} \frac{\sigma_i}{\sqrt{n_t}} & \text{(Assumption \ref{asm:unbiased-grad})}\\
 &\leq \frac{\eta G^\frac{1}{p-1}\norm{\vec{\sigma}}_1}{(p-1)\sqrt{n_t}}.
\end{align*}

Combining these two cases together (e.g., by the law of total expectation) completes the proof.
\end{proof}

\section{\texorpdfstring{$\ell_2$}{l2} Majorization and \texorpdfstring{$\ell_p$}{lp} Smoothness}\label{app:l2maj}
An assumption of interest, studied by~\citet{bernstein2018signsgd} (as well as~\citet{karimi2016linear}), is that of \emph{$\ell_2$ majorization} (with respect to $\vec{L} = [L_1, \dots, L_d]$), meaning that for all $x, y \in \R^d$, 
\begin{equation*}
    \abs{f(y) - f(x) - \grad f(x)^\top (y-x)} \leq \frac{1}{2}\sum\limits_{i=1}^d L_i(y^{(i)} - x^{(i)})^2.
\end{equation*}

We may equivalently express this condition as 1-smoothness w.r.t. $\norm{\cdot}_{\mathbf{L}}$, where $\mathbf{L} := \diag(\vec{L})$, i.e., for all $x, y \in \R^d$, $\norm{\grad f(y) - \grad f(x)}_{\mathbf{L}^{-1}} \leq \norm{y-x}_{\mathbf{L}}$.

Interestingly, we may observe that, for any $1 < \rho \leq \infty$ and letting $\rho^* := \frac{\rho}{\rho-1}$, we have
\begin{equation*}
     \frac{1}{\norm{\vec{L}}_{\rho^*}^{1/2}}\norm{\grad f(y) - \grad f(x)}_{2\rho/(2\rho-1)} \leq \norm{\grad f(y) - \grad f(x)}_{\mathbf{L}^{-1}} \leq \norm{y-x}_{\mathbf{L}} \leq \norm{\vec{L}}_{\rho^*}^{1/2}\norm{y-x}_{2\rho},
\end{equation*}
where the first inequality holds by reverse Hölder's inequality, i.e., for $u, v \in \R^d$, $\sum\limits_{i=1}^d |u^{(i)}v^{(i)}| \geq \norm{u}_{1/q}\norm{v}_{\frac{-1}{q-1}}$ (where we choose $q = \frac{2\rho-1}{\rho}$), and the last inequality holds by Hölder's inequality.

Rearranging, we have
    $\norm{\grad f(y) - \grad f(x)}_{2\rho/(2\rho-1)} \leq \norm{\vec{L}}_{\rho^*}\norm{y-x}_{2\rho}$,
and so it follows that, for $p > 2$, $\ell_2$ majorization implies $\norm{\vec{L}}_{\frac{p}{p-2}}$-smoothness w.r.t. $\norm{\cdot}_p$. Thus, while this condition is sufficient to entail $\ell_p$ smoothness (as previously noted by~\cite{balles2020geometry} in the case of $p = \infty$), we nevertheless prefer to work directly with $\ell_p$ smoothness assumptions, as we believe they provide a more natural pairing for the methods we consider.


\section{Hyperparameter Choices}\label{appendix:hyper}
We summarize the hyperparameters used in our experiments in Tables~\ref{tab:cifar_hyper_parameter}, \ref{tab:imagenet_hyper_parameter}, and \ref{tab:llm_hyper_parameter}. These hyperparameters are determined through a grid search. Specifically, we perform a search to identify appropriate values for the $\ell_p$-norm, learning rate $\eta$, $\alpha$, and weight decay $\lambda$. This process involves an initial rough comparison across a range of magnitudes, followed by a more precise grid search to determine the optimal values.

For fair comparison, all experimental settings, apart from the listed hyperparameters, follow the original papers of AdamW~\citep{loshchilovdecoupled} and Lion~\citep{chen2024symbolic}, and are kept consistent across all optimizers. For example, data augmentations for ImageNet~\citep{deng2009imagenet} and CIFAR~\citep{krizhevsky2009learning} all include random cropping and random horizontal flipping.

\begin{table}[ht]\scriptsize
\centering
\caption{CIFAR hyper-parameters.}
\label{tab:cifar_hyper_parameter}
\begin{tabular}{clcccccccccc}
\toprule
\textbf{Model} & \textbf{Optimizer}                & \textbf{Batch Size} & \textbf{$p$} & \textbf{Learning Rate} & \textbf{Schedule} & \textbf{$\alpha$} & \textbf{$\beta_1$} & \textbf{$\beta_2$} & \textbf{$\lambda$} & \textbf{$\tau$} & \textbf{$\epsilon$} \\ \midrule
ResNet-18      & SGD w/ Momentum                       & 128                 & -            & 0.02                   & cosine decay      & -                 & 0.9                & -                  & 0.0002             & -               & -                   \\
ResNet-18      & Adam~\citep{kingma2014adam}       & 128                 & -            & 0.001                  & cosine decay      & -                 & 0.9                & 0.999              & 0.0005             & -               & 1e-8                \\
ResNet-18      & AdamW~\citep{loshchilovdecoupled} & 128                 & -            & 0.01                   & cosine decay      & -                 & 0.9                & 0.999              & 0.0005             & -               & 1e-8                \\
ResNet-18      & Lion~\citep{chen2024symbolic}     & 128                 & -            & 0.001                  & cosine decay      & -                 & 0.9                & 0.99               & 0.01               & -               & -                   \\ \hline
ResNet-18      & {\sc Stacey}$_{(p,p)}$            & 128                 & 2            & 0.1                    & cosine decay      & 0.1               & 0.9                & 0.99               & 0.01               & 0.001           & 1e-12               \\
ResNet-18      & {\sc Stacey}$_{(p,2)}$            & 128                 & 2            & 0.1                    & cosine decay      & 0.1               & 0.9                & 0.99               & 0.01               & 0.001           & 1e-12               \\ \bottomrule
\end{tabular}
\end{table}

\begin{table}[ht]\scriptsize
\centering
\caption{ImageNet hyper-parameters.}
\label{tab:imagenet_hyper_parameter}
\begin{tabular}{clcccccccccc}
\toprule
\textbf{Model}                & \textbf{Optimizer}                & \textbf{Batch Size} & \textbf{$p$} & \textbf{Learning Rate} & \textbf{Schedule} & \textbf{$\alpha$} & \textbf{$\beta_1$} & \textbf{$\beta_2$} & \textbf{$\lambda$} & \textbf{$\tau$} & \textbf{$\epsilon$} \\ \midrule
ResNet-50                     & SGD w/ Momentum                       & 256                 & -            & 0.01                   & cosine decay      & -                 & 0.9                & -                  & 0.0005             & -               & -                   \\
\multicolumn{1}{l}{ResNet-50} & AdamW~\citep{loshchilovdecoupled} & 256                 & -            & 0.002                  & cosine decay      & -                 & 0.9                & 0.999              & 0.005              & -               & 1e-4                \\
\multicolumn{1}{l}{ResNet-50} & Lion~\citep{chen2024symbolic}     & 256                 & -            & 3e-4                   & cosine decay      & -                 & 0.9                & 0.99               & 0.01               & -               & -                   \\ \hline
ResNet-50                     & {\sc Stacey}$_{(p,p)}$            & 256                 & 3            & 0.01                   & cosine decay      & 0.001               & 0.9                & 0.999              & 0.001             & 0.001           & 1e-8                \\
ResNet-50                     & {\sc Stacey}$_{(p,2)}$            & 256                 & 2.8          & 0.01                   & cosine decay      & 0.001              & 0.9                & 0.999              & 0.001             & 0.001           & 1e-8                \\ \bottomrule
\end{tabular}
\end{table}

\begin{table}[ht]\scriptsize
\centering
\caption{Hyper-parameters for LLM pretraining.}
\label{tab:llm_hyper_parameter}
\begin{tabular}{clccclcccccc}
\hline
\textbf{Model}      & \textbf{Optimizer}                & \textbf{Batch Size} & \textbf{$p$} & \textbf{Learning Rate} & \multicolumn{1}{c}{\textbf{Schedule}} & \textbf{$\alpha$} & \textbf{$\beta_1$} & \textbf{$\beta_2$} & \textbf{$\lambda$} & \textbf{$\tau$} & \textbf{$\epsilon$} \\ \hline
\texttt{llama-100m} & SGD w/ Momentum                       & 16                  & -            & 0.01                   & cosine decay                          & -                 & 0.9                & -                  & 0.0005             & -               & -                   \\
\texttt{llama-100m} & Adam~\citep{kingma2014adam}       & 16                  & -            & 0.0001                 & cosine decay                          & -                 & 0.9                & 0.999              & 0.01               & -               & 1e-8                \\
\texttt{llama-100m} & AdamW~\citep{loshchilovdecoupled} & 16                  & -            & 0.0001                 & cosine decay                          & -                 & 0.9                & 0.999              & 0.05               & -               & 1e-8                \\
\texttt{llama-100m} & Lion~\citep{chen2024symbolic}     & 16                  & -            & 0.05                   & cosine decay                          & -                 & 0.9                & 0.999              & 0.01               & -               & -                   \\ \hline
\texttt{llama-100m} & {\sc Stacey}$_{(p,p)}$            & 16                  & 3            & 0.01                   & cosine decay                          & 0.1              & 0.9                & 0.99               & 0.01               & 0.001           & 1e-8                \\
\texttt{llama-100m} & {\sc Stacey}$_{(p,2)}$            & 16                  & 2.8          & 0.01                   & cosine decay                          & 0.1              & 0.9                & 0.99               & 0.0005             & 0.001           & 1e-8                \\ \hline
\end{tabular}
\end{table}


\end{document}